\documentclass[letterpaper,12pt]{article}

\usepackage[utf8]{inputenc} 
\usepackage[T1]{fontenc} 
\usepackage[english]{babel}
\usepackage{booktabs}       
\usepackage{amsfonts}       
\usepackage{nicefrac}       
\usepackage{microtype}      
\usepackage{amsmath,amsthm}
\usepackage{mathtools}
\usepackage{mleftright}
\usepackage{graphicx}
\usepackage{multirow}

\usepackage[left=1in,right=1in,top=1in,bottom=1in,letterpaper]{geometry}

\usepackage{caption} 
\usepackage{dsfont}
\usepackage{amssymb}
\usepackage{float}
\usepackage{amsmath}
\usepackage[noend]{algpseudocode}
\usepackage{hyperref}
\usepackage{algorithm,algpseudocode}
\usepackage{array}
\usepackage{url}
\usepackage{subcaption}
\usepackage{natbib}
\usepackage{xcolor}
\setcitestyle{square,numbers}

\DeclarePairedDelimiterX{\expectarg}[1]{[}{]}{%
  \ifnum\currentgrouptype=16 \else\begingroup\fi
  \activatebar#1
  \ifnum\currentgrouptype=16 \else\endgroup\fi
}

\newcommand{\LinesNumbered}{
  \setboolean{algocf@linesnumbered}{true}%
  \renewcommand{\algocf@linesnumbered}{\everypar={\nl}}}%
\let\oldnl\nl
\newcommand{\nonl}{\renewcommand{\nl}{\let\nl\oldnl}}

\newcommand\tab[1][1cm]{\hspace*{#1}}
\newcommand{\quotes}[1]{``#1''}

\algnewcommand{\Inputs}[1]{%
  \State \textbf{Inputs:}
  \Statex \hspace*{\algorithmicindent}\parbox[t]{.8\linewidth}{\raggedright #1}
}
\algnewcommand{\Initialize}[1]{%
  \State \textbf{Initialize:}
  \Statex \hspace*{\algorithmicindent}\parbox[t]{.8\linewidth}{\raggedright #1}
}
\algnewcommand{\Data}[1]{%
  \State \textbf{Data:}
  \Statex \hspace*{\algorithmicindent}\parbox[t]{.8\linewidth}{\raggedright #1}
}

\DeclareMathOperator*{\argmin}{\arg\!\min}
\newtheorem{theorem}{Theorem}
\newtheorem{lem}{Lemma}

\graphicspath{{images/}} 

\title{Convergence Analyses of Online ADAM Algorithm in Convex Setting and Two-Layer ReLU  Neural Network}

\author{
  Biyi Fang\\
  Department of Engineering Science and Applied Mathematics\\
  Northwestern University\\
  Evanston, IL 60208 \\
  \texttt{biyifang2021@u.northwestern.edu} \\
  \and
  Diego Klabjan \\
  Department of Industrial Engineering and Management Sciences \\
  Northwestern University\\
  Evanston, IL 60208 \\
  \texttt{d-klabjan@northwestern.edu} \\
}

\begin{document}
\maketitle
\begin{abstract}
Nowadays, online learning is an appealing learning paradigm, which is of great interest in practice due to the recent emergence of large scale applications such as online advertising placement and online web ranking. Standard online learning assumes a finite number of samples while in practice data is streamed infinitely. In such a setting gradient descent with a diminishing learning rate does not work. We first introduce regret with rolling window, a performance metric for online streaming learning, which measures the performance of an algorithm on every fixed number of contiguous samples. At the same time, we propose a family of algorithms based on gradient descent with a constant or adaptive learning rate and provide very technical analyses establishing regret bound properties of the algorithms. We cover the convex setting showing the regret of the order of the square root of the size of the window in the constant and dynamic learning rate scenarios. Our proof is applicable also to the standard online setting where we provide analyses of the same regret order (the previous proofs have flaws). We also study a two layer neural network setting with reLU activation. In this case we establish that if initial weights are close to a stationary point, the same square root regret bound is attainable. We conduct computational experiments demonstrating a superior performance of the proposed algorithms. 

\end{abstract}

\section{Introduction}
Online learning is the process of dynamically incorporating knowledge of the geometry of the data observed in earlier iterations to perform more informative learning in later iterations, as opposed to standard machine learning techniques which provide an optimal predictor after training over the entire dataset.  Online learning is a preferred paradigm in situations where the algorithm has to dynamically adapt to new patterns in the dataset, or when the dataset itself is generated as a function of time, i.e. stock price prediction. Online learning is also used when the dataset itself is computationally infeasible to be trained over the entire dataset.

In standard online learning it is assumed that a finite number of samples is encountered however in real world streaming setting an infinite number of samples is observed (e.g., Twitter is streaming since inception and will continue to do so for foreseeable future). 
The performance of an online learning algorithm on early examples is negligible when measuring the performance or making predictions and decisions on the later portion of a dataset (the performance of an algorithm on tweets from ten years ago has very little bearing on its performance on recent tweets). The problem can be tackled by restarting, however, it is challenging to determine when to restart. For this reason we propose a metric which forgets about samples encountered a long time ago. Consequently, we introduce a performance metric, regret with rolling window, which measures the performance of an online learning algorithm over a possible infinite size dataset. 
This metric also requires an adaptation of prior algorithms, because, for example, a diminishing learning rate has poor performance on an infinite data stream.

Stochastic gradient descent ({\scshape Sgd}) \cite{zinkevich2003online} is a widely used approach in areas of online machine learning, where the weights are updated each time a new sample is received. Furthermore, it requires a diminishing learning rate in order to achieve a high-quality performance. It has been empirically observed that, in order to reduce the impact of the choice of the learning rate and conduct stochastic optimization more efficiently, the adaptive moment estimation algorithm ({\scshape Adam}) \cite{kingma2014adam} and its extensions (\cite{luo2019adaptive},\cite{reddi2018convergence}) are another type of popular methods, which store an exponentially decaying average of past gradients and squared gradients and applies adaptive learning rate. (In standard gradient descent algorithms we use the term learning rate, while in adaptive learning rate algorithms we call stepsize the hyperparameter that governs the scale between the weights and the adjusted gradient.) In spite of this, no contribution has been made to the case where the regret is computed in a rolling window. Moreover, applying a diminishing learning rate or stepsize to regret with rolling window is not a good strategy, otherwise, the performance is heavily dependent on the learning rate or stepsize and the rank of a sample. Namely, regret with rolling window requires a constant learning rate or stepsize.

Standard online setting has been studied in the convex setting. 
With improvements in computational power resulting from GPUs, deep neural networks have been very popular in a variety of AI problems recently. A core application of online learning is online web search and recommender systems \cite{zoghi2017online} where deep learning solutions have recently emerged. At the same time, online learning based on deep neural networks has become an integral role in many stages in finance, from portfolio management, algorithmic trading, to fraud detection, to loan and insurance underwriting. To this end we focus not only on convex loss functions, but also on deep neural networks.  

In this work, we propose a new family of efficient online  subgradient methods for both general convex functions and a two-layer ReLU neural network based on regret with rolling window metric. More precisely, we first present an algorithm, namely convergent  Adam ({\scshape convgAdam}), designed for general strictly convex functions based on gradient descent using adaptive learning rate and inspired by the work of \cite{reddi2018convergence}. {\scshape convgAdam} is a more general algorithm that can dynamically adapt to an arbitrary sequence of strictly convex functions. In the meanwhile, we experimentally show that {\scshape convgAdam} outperforms state-of-the-art, yet non-adaptive, online gradient descent ({\scshape OGD}) \cite{zinkevich2003online}. Then, we propose an algorithm, called deep neural network gradient descent ({\scshape dnnGd}), for a two-layer ReLU neural network. {\scshape dnnGd} takes standard gradient first, then it rescales the weights upon receiving each new sample. Lastly, we introduce a new algorithm, deep neural network Adam ({\scshape dnnAdam}), which uses an adaptive learning rate for the two-layer ReLU neural network. {\scshape dnnAdam} is first endowed with long-term memory by using gradient updates scaled by square roots of exponential decaying moving averages of squared past gradients and then it rescales weights with every new sample.

In this paper, we not only propose a new family of gradient-based online learning algorithms for both convex and non-convex loss functions, but also present a complete technical proof of regret with rolling window for each of them. For strongly convex functions, given a constant learning rate, we show that {\scshape convgAdam} attains regret with rolling window which is proportional to the square root of the size of the rolling window, compared to the true regret $\mathcal{O}(\log(T)\sqrt{T})$ of {\scshape AMSGrad} \cite{reddi2018convergence}. Besides, we not only point out but also fix the problem in the proof of regret for {\scshape AMSGrad} later in this paper. Table \ref{resultsummary} in Appendix \hyperref[correction]{A.2} summarizes all regret bounds in various settings, including the previous flawed analyses. Furthermore, we prove that both {\scshape dnnGd} and {\scshape dnnAdam} attain the same regret with rolling window under reasonable assumptions for the two-layer ReLU neural network. The strongest assumption requires that the angle between the current sample and weight error is bounded away from $\pi/2$. Although {\scshape dnnGd} and {\scshape dnnAdam} require some assumptions, these two algorithms have a higher probability to converge than other flavors of {\scshape Adam} due to the convergence analyses provided in Section 5.
In summary, we make the following five contributions.
\begin{itemize}
\item We introduce regret with rolling window that is applicable in data streaming, i.e., infinite stream of data. 
\item We provide a proof of regret with rolling window which is proportional to the square root of the size of the rolling window when applying {\scshape OGD} to an arbitrary sequence of convex loss functions.
\item We provide a convergent first-order gradient-based algorithm, i.e. {\scshape convgAdam}, employing adaptive learning rate to dynamically adapt to the new patterns in the dataset. Furthermore, given strictly convex functions and a constant stepsize, we provide a complete technical proof of regret with rolling window. Besides, we point out a problem with the proof of convergence of {\scshape AMSGrad} \cite{reddi2018convergence}, which eventually leads to $\mathcal{O}(\log(T)\sqrt{T})$ regret in the standard online setting, and we provide a different analysis for {\scshape AMSGrad} which obtains $\mathcal{O}(\sqrt{T})$ regret in standard online setting by using our proof technique. 
\item We propose a first-order gradient-based algorithm, called {\scshape dnnGd}, for the two-layer ReLU neural network. Moreover, we show that {\scshape dnnGd} shares the same regret with rolling window with {\scshape OGD} when employing a constant learning rate. 
\item We further develop an algorithm, i.e. {\scshape dnnAdam}, based on adaptive estimation of lower-order moments for the two-layer ReLU neural network. At the same time, we argue that {\scshape dnnAdam} shares the same regret with rolling window with {\scshape convgAdam} when employing a constant stepsize.
\item We present numerical results showing that {\scshape convgAdam} outperforms state-of-art, yet not adaptive, {\scshape OGD}.
\end{itemize}
 
The paper is organized as follow. In the next section, we review several works related to {\scshape Adam}, analyses of two-layer neural networks and regret in online convex learning. In Section 3, we state the formal optimization problem in streaming, i.e., we introduce regret with rolling window. In the subsequent section we propose the two algorithms in presence of convex loss functions and we provide the underlying regret analyses. In Section 5 we study the case of deep neural networks as the loss function.  In Section 6 we present experimental results comparing {\scshape convgAdam} with {\scshape OGD}. 

\section{Related Work}
\textit{\textbf{{\scshape Adam} and its variants:}} {\scshape Adam} \cite{kingma2014adam} is one of the most popular stochastic optimization methods that has been applied to convex loss functions and deep networks which is based on using gradient updates scaled by square roots of exponential moving averages of squared past gradients. In many applications, e.g. learning with large output spaces, it has been empirically observed that it fails to converge to an optimal solution or a critical point in nonconvex settings. A cause for such failures is the exponential moving average, which leads {\scshape Adam} to forget about the influence of large and informative gradients quickly \cite{chen2018convergence}. To tackle this issue, {\scshape AMSGrad} \cite{reddi2018convergence} is introduced which has long-term memory of past gradients. {\scshape AdaBound} \cite{luo2019adaptive} is another extension of {\scshape Adam}, which employs dynamic bounds on learning rates to achieve a gradual and smooth transition from adaptive methods to stochastic gradient. Though both {\scshape AMSGrad} \cite{reddi2018convergence} and {\scshape AdaBound} \cite{luo2019adaptive} provide theoretical proofs of convergence in a convex case,  very limited further research related to {\scshape Adam} has be done in a non-convex case while {\scshape Adam} in particular has become the default algorithm leveraged across many deep learning frameworks due to its rapid training loss progress. Unfortunately, there are flaws in the proof of {\scshape AMSGrad}, which is explained in a later section and articulated in Appendix A.\vspace{0.2cm}\\
\textit{\textbf{Two-layer neural network:}} Deep learning achieves state-of-art performance on a wide variety of problems in machine learning and AI. Despite its empirical success, there is little theoretical evidence to support it. Inspired by the idea that gradient descent converges to minimizers and avoids any poor local minima or saddle points (\cite{lee2016gradient}, \cite{lee2017first},  \cite{baldi1989neural}, \cite{goodfellow2016deep}, \cite{kawaguchi2016deep}), Luo \& Wu \cite{wu2018no} prove that there is no spurious local minima in a two-hidden-unit ReLU network. However, Luo \& Wu make an assumption that the 2$^{nd}$ layer is fixed, which does not hold in applications. Li \& Yuan \cite{li2017convergence} also make progress on understanding algorithms by providing a convergence analysis for {\scshape Sgd} on special two-layer feedforward networks with ReLU activations, yet, they specify the 1$^{st}$ layer as begin offset by \quotes{identity mapping} (mimicking residual connections) and the 2$^{nd}$ layer as the $\ell_1$-norm function. Additionally, based on their work \cite{du2017convolutional}, Du et al \cite{du2017gradient} give the 2$^{nd}$ layer more freedom in the problem of learning a two-layer neural network with a non-overlapping convolutional layer and ReLU activation. They prove that although there is a spurious local minimizer, gradient descent with weight normalization can still recover the true parameters with constant probability when given Gaussian inputs. Nevertheless, the convergence is guaranteed when the 1$^{st}$ layer is a convolutional layer. \vspace{0.2cm}\\
\textit{\textbf{Online convex learning:}} Many successful algorithms and associated proofs have been studied and provided over the past few years to minimize regret in online learning setting. Zinkevich \cite{zinkevich2003online} shows that the online gradient descent algorithm achieves regret $\mathcal{O}(\sqrt{T})$, for an arbitrary sequence of $T$ convex loss functions (of bounded gradients) and given a diminishing learning rate. Then, Hazan et al \cite{Hazan:2007:LRA:1296038.1296051} improve regret to $\mathcal{O}(
\log(T))$ when given strictly convex functions and a diminishing learning rate. The idea of adapting first order optimization methods is by no means new and is also popular in online convex learning. Duchi, Hazan \& Singer \cite{duchi2011adaptive} present {\scshape AdaGrad}, which employs very low learning rates for frequently occurring features and high learning rates for infrequent features, and obtain a comparable bound by assuming 1-strongly convex proximal functions.  In a similar framework, Zhu \& Xu \cite{Zhu2015OnlineGD} extend the celebrated online gradient descent algorithm to Hilbert spaces (function spaces) and analyzed the convergence guarantee of the algorithm. The online functional gradient algorithm they propose also achieves regret $\mathcal{O}(\sqrt{T})$ when given convex loss functions. In all these algorithms, the loss function is required to be convex or strongly convex and the learning rate or step size must diminish. However, no work about regret analyses of online learning applied on deep neural networks (non-convex loss functions) has been done.\\
\textit{\textbf{Adaptive regret:}} Recently, adaptive regret has been studied in the setting of prediction with expert advice (PEA) in online learning. Adaptive regret measures the maximum difference of the performances of an online algorithm and the offline optimum for any consecutive $\tau$ samples in total $T$ rounds, while our regret measures the maximum difference of the performances in the whole history. Existing online algorithms are closely related in the sense that adaptive algorithms designed are usually built upon the PEA algorithms. The concept of adaptive regret is formally introduced by Hazan and Seshadhri \cite{hazan2007adaptive}. They also propose a new algorithm named follow the leading history (FLH), which contains an expert-algorithm, a set of intervals and a meta-algorithm. Then Daniely in \cite{daniely2015strongly} extends this idea by introducing strongly adaptive algorithms, which provide a regret bound $\mathcal{O}(\log(s+1)\sqrt{\left|I\right|}$ where s is the end point of the rolling window and $|I|$ is the size of the window. Later on, Zhang in \cite{pmlr-v97-zhang19j} applies the concept of the adaptive learning rate to adaptive regret and proposes adaptive algorithms for convex and smooth functions, and finally obtains a regret bound $\mathcal{O}(\sqrt{\left(\sum_{t=r}^sf_t(\omega) \right)\log(s)\log(s-r)})$, where $f_t(\omega)$ is the loss function for the $t\textit{'th}$ sample given any $\omega$ in the corresponding domain, and $r$ and $s$ are the starting and ending data points of the interested sequence. Notice that in conjunction with infinite streaming, $s$ blows up and eventually dominates the regret bound. Although the concept of the adaptive regret is similar to our regret with rolling window, adaptive regret relies on other existing online algorithms which not only use diminishing learning rate but also bring extra error. In regret with rolling window, we consider these two aspects together (infinite time stream and the issue of learning rates) and propose a new family of online algorithms which use a constant learning rate and achieve a more robust regret. Specifically, our regret bound is $\mathcal{O}(\sqrt{T})$, which does not depend of the position of the window.

\section{Regret with Rolling Window}

We consider the problem of optimizing regret with rolling window, inspired by the standard regret (\cite{zinkevich2003online}, \cite{abernethy2012interior}, \cite{rakhlin2009lecture}). The problem with the traditional regret is that it captures the performance of an algorithm only over a fixed number of samples or loss functions. In most applications data is continuously streamed with an infinite number of future loss functions. The performance over any finite number of consecutive loss functions $T$ is of interest. The concept of regret is to compare the optimal offline algorithm with access to $T$ contiguous loss functions with the performance of the underlying online algorithm. Regret with rolling window is to find the maximum of all differences between the online loss and the loss of the offline algorithm for any $T$ contiguous samples. More precisely, for an infinite sequence $\left\{z^t,y^t\right\}_{t=1}^{\infty}$, where each feature vector $z^t\in\mathbb{R}^d$ is associated with the corresponding label $y^t$, given fixed $T$ and any $p$, we first define $\omega_p^*\in \argmin_{\omega}\sum_{t=p}^{p+T}f_t(\omega)$, which corresponds to the optimal solution of the offline algorithm. In general, $f_t(\omega)=\mathrm{loss}(x^t,y^t;\omega)$, e.g. $f_t(\omega)=\left\|\omega^Tx^t-y^t\right\|^2$ if the linear regression model is applied and the mean square error is used.
Then, we consider
\begin{align}
    \max_{p\in\mathbb{N}} R_p(T):=\sum_{t=p}^{T+p} l_t(\omega_t)
\label{goal}
\end{align}
with $l_t(\omega_t) =f_t(\omega_t)-f_t(\omega_p^*) $, where $f_t$ is a function of sample $z^t$. The regret with rolling window metric captures regret over every $T$ consecutive loss functions and it is aiming to assess the worst possible regret over every such sequence. Note that if we have only $T$ loss functions corresponding only to $p=1$, then this is the standard regret definition in online learning. The goal is to develop algorithms with low regret with rolling window. We prove that regret with rolling window can be bounded by $\mathcal{O}(\sqrt{T})$. In other words, the average regret with rolling window approaches zero.

\section{Convex Setting}

In the convex setting, we propose two algorithms with a different learning rate or stepsize strategy and analyze them with respect to (\ref{goal}) in the streaming setting.

\subsection{Algorithms}

Algorithms in standard online setting are almost all based on gradient descent where the parameters are updated after each new loss function is received based on the gradient of the current loss function. A challenge is the strategy to select an appropriate learning rate. In order to guarantee good regret the learning rate is usually decaying.  
In the streaming setting, we point out that a decaying learning rate is improper since far away samples (very large $p$) would get a very small learning rate implying low consideration to such samples. In conclusion, the learning rate has to be a constant or it should follow a dynamically adaptive learning algorithm, i.e. {\scshape ADAM}. The algorithms we provide for solving (\ref{goal}) in the streaming setting are based on gradient descent and one of the just mentioned learning rate strategies. 

In order to present our algorithms, we first need to specify notation and parameters. In each algorithm, we denote by $\eta$ and $g_t$ the learning rate or stepsize and a subgradient of loss function $f_t$ associated with sample $(z^t,y^t)$, respectively. Additionally, we employ $\odot$ to represent the element-wise multiplication between two vectors or matrices (Hadamard product). However, for other operations we do not introduce new notation, e.g., for element-wise division ($/$) and square root ($\sqrt{\mathrm{\,\,}}$), since these two operations are written differently when representing standard matrix or vector operations. 


We start with {\scshape OGD} which mimics gradient descent in online setting and achieves $\mathcal{O}(\sqrt{T})$ regret with rolling window when given constant learning rate. Algorithm \ref{SGD} is a twist on Zinkevich's {\scshape Online Gradient Descent} \cite{zinkevich2003online}. {\scshape OGD} updates its weight when a new sample is received in step~\ref{SGDupdate}. In addition, {\scshape OGD} uses constant learning rate in the streaming setting so as to efficiently and dynamically learn the geometry of the dataset. Otherwise, if a diminishing learning rate is applied, {\scshape OGD} misses informative samples which arrive late due to the extremely small learning rate and leads to $\mathcal{O}(T)$ regret with rolling window (this is trivial to observe if the loss functions are bounded). Regret of $\mathcal{O}(\sqrt{T})$ is achieved in the streaming setting if learning rate $\eta=1/\sqrt{T}$.

Constant learning rates have a drawback by treating all features equally. Consequently, we adapt {\scshape ADAM} to online setting and further extend it to streaming. Algorithm \ref{convgADAM}, inspired by {\scshape ADAM} \cite{kingma2014adam} and {\scshape AMSGrad} \cite{reddi2018convergence}, has regret with rolling window also of the order $\mathcal{O}(\sqrt{T})$ given constant stepsize $\eta$ as shown in the next section. The key difference of {\scshape convgADAM} with {\scshape AMSGrad} is that it maintains the same ratio of the past gradients and the current gradient instead of putting more and more weight on the current gradient and losing the memory of the past gradients fast. In Algorithm \ref{convgADAM}, {\scshape convgAdam} records exponential moving average of gradients and moments in step~\ref{convg:m} and~\ref{ADAM:vupdate}. Step~\ref{ADAm:max} guarantees that $\hat{v}_t$ is the maximum of all $v_t$ until the present time step. Then, step~\ref{ADAM:update} gives the adaptive update rule by using the maximum value of $v_t$ to normalize the running average of the gradient at time $t$. Besides, constant stepsize $\eta$ is crucial to make {\scshape convgAdam} well-performed due to the aforementioned reason with a potential decaying learning rate or stepsize. 

In step~\ref{ADAM:update}, we observe that $\hat{v}_{ti} = 0$ for a feature $i$ implies $m_{ti} = 0$, therefore, we retain the foregoing weight $\omega_{ti}$ as the succeeding weight $\omega_{t+1,i}$. In other words, in this case we define $\frac{\eta}{\sqrt{\hat{v}_{ti}}}\cdot m_{ti} = \frac{\eta}{\sqrt{0}}\cdot 0 = 0$.

\begin{minipage}[t]{0.46\textwidth}
\vspace{0pt} 
\begin{algorithm}[H]
\caption{{\scshape Online Gradient Descent}}
\label{SGD}
\begin{algorithmic}[1]
\State $\textit{Positive\,parameter\,\,} \eta$
\For{$t=0,1,2,\cdots$}
\State $g_t=\bigtriangledown f_t(w_t)$
\State $w_{t+1} = w_t-\eta g_t$\label{SGDupdate}
\EndFor
\State \textbf{end for}
\end{algorithmic}
\end{algorithm}
\end{minipage}
\hfill
\begin{minipage}[t]{0.46\textwidth}
\vspace{0pt} 
\begin{algorithm}[H]
\caption{{\scshape convergent Adam}}
\label{convgADAM}
\begin{algorithmic}[1]
\State $\textit{Positive parameters\,\,} \eta,\beta_1<1,\beta_2<1$
\State $\textit{Set\,\,} m_0=0, v_0=0, \text{\,\,and\,\,}\hat{v}_0=0$
\For{$t=0,1,2,\cdots,$}
\State  $g_t=\bigtriangledown f_t(w_t)$\label{adam:s}
\State  $m_t=\beta_1m_{t-1}+(1-\beta_1)g_t$\label{convg:m}
\State  $v_t=\beta_2v_{t-1}+(1-\beta_2)g_t\odot g_t$\label{ADAM:vupdate}
\State  $\hat{v}_t=\mathrm{max}(\hat{v}_{t-1},v_t)$
\label{ADAm:max}
\State $w_{t+1} = w_t-\frac{\eta}{\sqrt{\hat{v}_t}}\odot m_t$\label{ADAM:update}
\EndFor
\State \textbf{end for}
\end{algorithmic}
\end{algorithm}
\end{minipage}

\subsection{Analyses}

In this section, we provide regret analyses of {\scshape OGD} and {\scshape convgAdam} showing that both of them attain regret with rolling window which is proportional to the square root of the size of the rolling window given a constant learning rate or stepsize in the streaming setting. For inner (scalar) products, given the fact that $\left\langle a,b \right\rangle = a^Tb=b^Ta$ for two vectors $a$ and $b$, in the rest of the paper, for short expressions we use $a^Tb$, but for longer we use $\left\langle a,b \right\rangle$.

We require the following standard assumptions. 
\subsubsection*{Assumption 1:}
\label{as:1}
\begin{enumerate}
\item \label{as1:1} 
There exists a constant $D_{\infty}$, such that
$\left\|\omega_t\right\|_{\infty}\leq \frac{D_{\infty}}{2}$, 
for any $t\in\mathbb{N}$. 
\item \label{as1:2} The loss gradients $\triangledown f_t(\omega_t)$ are bounded, i.e., for all $\omega_t$ such that $\left\|\omega_t\right\|_{\infty}\leq \frac{D_{\infty}}{2}$, we have
$\left\| \triangledown f_t(\omega_t) \right\|_{\infty}\leq G_{\infty}$.
\item \label{as1:3} Functions $f_t(\omega)$ are convex and differentiable with respect to $\omega$ for every $t\in\mathbb{N}$.
\item \label{as1:4} Functions $f_t(\omega)$ are strongly convex with parameter $H$, i.e., for all $\bar{\omega}$ and $\tilde{\omega}$, and for $t\in\mathbb{N}$, it holds
$f_t(\bar{\omega})+\bigtriangledown f_t(\bar{\omega})(\tilde{\omega}-\bar{\omega})+\frac{H}{2}\left\|\tilde{\omega}-\bar{\omega} \right\|^2\leq f_t(\tilde{\omega})$.
\end{enumerate}

The first condition in Assumption \hyperref[as:1]{1} assumes that $\omega_t$ are bounded. This assumption can be removed by further complicating certain aspects of the upcoming proofs. This extension is discussed in Appendix \hyperref[projection]{A.1} for the sake of clarity of the algorithm. In \ref{as1:2} from Assumption \hyperref[as:1]{1}, the gradient of the loss function is requested to be upper bounded. Notice that each loss function $f_t$ is enforced to be differentiable and convex in \ref{as1:3}, whereas $f_t$ is required to be strictly convex with parameter $H$ in \ref{as1:4}. All these are standard assumptions. 

We first provide the regret analysis of {\scshape OGD}. 

\begin{theorem}
\label{thm1}
If \ref{as1:1}-\ref{as1:3} in Assumption \hyperref[as:1]{1} hold, and $\eta=\frac{\eta_1}{\sqrt{T}}$ for any positive constant $\eta_1$, the sequence $\omega_t$ generated by {\scshape OGD} achieves $\max_{p\in\mathbb{N}}R_p(T)\leq \mathcal{O}(\sqrt{T})$.
\end{theorem}

The proof is provided in Appendix \hyperref[proof:thm1]{B}.
Under the assumptions in Assumption \hyperref[as:1]{1}, by finding a relationship for the sequence of the weight error $\omega_t-\omega_p^*$ and employing the property of convexity from condition \ref{as1:3} from Assumption \hyperref[as:1]{1}, we prove that {\scshape OGD} obtains the regret with rolling window which is proportional to the square root of the size of the rolling window when given the constant learning rate. This is consistent with the regret of {\scshape OGD} in the standard online setting. 

The analysis of {\scshape OGD} is not totally new but still has some important differences. Also, when using a diminishing learning rate, given strongly convex $f_t$, it has been proven in \cite{Hazan:2007:LRA:1296038.1296051} by Hazan that OGD obtains logarithmic regret. However, this is not possible even given strongly convex $f_t$ when using a constant learning rate, which should be clear after reading our regret analysis of {\scshape OGD} and comparing it with the regret analysis in \cite{Hazan:2007:LRA:1296038.1296051}.

If {\scshape OGD} is an analogue to the Gradient Descent optimization method for the online setting, then {\scshape convgAdam} is an online analogue of {\scshape Adam}, which dynamically incorporates knowledge of the characteristics of the dataset received in earlier iterations to perform more informative gradient-based learning. Next, we show that {\scshape convgAdam} achieves the same regret with rolling window given a constant stepsize.

\begin{theorem}
\label{thm2}
If Assumption \hyperref[as:1]{1} holds, and $\beta_1$ and $\beta_2$ are two constants between 0 and 1 such that $\lambda:=\frac{\beta_1}{\sqrt{\beta_2}}<1$ and $\beta_1\leq \frac{H\eta}{1+H\eta}$, then for $\eta=\frac{\eta_1}{\sqrt{T}}$ for any positive constant $\eta_1$, the sequence $\omega_t$ generated by {\scshape convgAdam} achieves $\max_{p\in\mathbb{N}}R_p(T)\leq \mathcal{O}(\sqrt{T})$.
\end{theorem}

The proof is provided in  Appendix \hyperref[proof:thm2]{C}.
The very technical proof follows the following steps. Based on the updating procedure in steps~\ref{adam:s}-\ref{ADAM:update}, we establish a relationship for the sequence of the weight error $\omega_t-\omega_p^*$. Meanwhile, considering condition \ref{as1:4} in Assumption \hyperref[as:1]{1}, we obtain another relationship between the loss function error $f_t(\omega_t)-f_t(\omega_p^*)$ and $\left\langle \omega_t-\omega_*,\bigtriangledown f_t(\omega_t)\right\rangle$. Assembling these two relationships provide a relationship between the weight error $\omega_t-\omega_*$ and the loss function error $f_t(\omega_t)-f_t(\omega_*)$. By deriving upper bounds for all of the remaining terms based on conditions from Assumption \hyperref[as:1]{1}, we are able to argue the same regret with rolling window of $\mathcal{O}(\sqrt{T})$. 

In the regret analysis of {\scshape AMSGrad} \cite{reddi2018convergence}, the authors forget that the stepsize is $\frac{1}{\sqrt{t}}$ and take the hyperparameter $\eta$ to be exponentially decaying for granted without assumptions which eventually leads to $\mathcal{O}(T\sqrt{T})$ regret in standard online setting. Our analysis is flexible enough to extend to {\scshape AMSGrad} and a slight change to our proof yields the $\mathcal{O}(\sqrt{T})$ regret for {\scshape AMSGrad}. The changes in our proof to accommodate standard online setting and {\scshape AMSGrad} are stated in Appendix \hyperref[correction]{A.2}. Moreover, the proof of convergence of {\scshape AMSGrad} in \cite{reddi2018convergence} uses a diminishing stepsize while our proof is valid for both constant and diminishing stepsizes. Likewise, for {\scshape AdaBound} \cite{luo2019adaptive}, the right scale of the stepsize is also missed and the regret should be $\mathcal{O}(T)$, which is discussed in more detail in Appendix \hyperref[correction]{A.2}. In this section we also discuss how to amend our proof to provide the $\mathcal{O}(\sqrt{T})$ regret bound in standard online setting for {\scshape AdaBound}.

Theorem \ref{thm2} guarantees that {\scshape convgAdam} achieves the same regret with rolling window as {\scshape OGD} for convex loss functions. On the other hand, very limited work has been done about regret for nonconvex loss functions, e.g. the loss function of a two-layer ReLU neural network. In the following section, we argue that both {\scshape dnnGD} and {\scshape dnnAdam} attain the same regret with rolling window if the initial starting point is close to an optimal offline solution and by using a constant learning rate or stepsize. In addition to a favorable starting point, further assumptions are needed. 

\section{Two-Layer ReLU Neural Network}

In this section we consider a two layer neural network with the first hidden layer having an arbitrary number of neurons and the second hidden layer has a single neuron. The underlying activation function is a probabilistic version of ReLU and minimum square error is considered as the loss function. 
First of all, the optimization problem of such a two-layer ReLU neural network is neither convex nor convex (and clearly non linear), therefore, it is very hard to find a global minimizer. Instead, we show that our algorithms achieve $\mathcal{O}(\sqrt{T})$ regret with rolling window when the initial point is close enough to an optimal solution.

Neural networks as classifiers have a lot of success in practice, whereas a formal theoretical understanding of the mechanism of why they work is largely missing. Studying a general neural network is challenging, therefore, we focus on the proposed two-layer ReLU neural network. For a dataset $\left\{z^t, y^t \right\}_{t=1}^{\infty}$, the standard loss function of the two-layer neural network is $f_t\left(\omega_{1,t},\omega_{2,t}\right)=\frac{\left\|\omega_{1,t}^T\sigma\left(\omega_{2,t}z^t\right)-y^t\right\|^2}{2}$,
where $\sigma$ represents the ReLU activation function applied element-wise, $\omega_{1,t}$ is the parameter vector, and $\omega_{2,t}$ is the parameter matrix. It turns out that ReLU is challenging to analyze since nesting them yields many combinations of the various values being below zeros. One way to get around this is to consider a probabilistic version of ReLU and capturing expected loss, Luo \& Wu \cite{wu2018no}.

To this end we treat ReLU as a random Bernoulli variable in the sense that Pr$(\sigma(x)=x) = \rho$, Pr$(\sigma(x)=0)=1-\rho$. 
Luo \& Wu \cite{wu2018no} in the standard offline setting analyze $f_t\left(\omega_{1,t},\omega_{2,t}\right)$ for the probabilistic version of ReLU. For our online analyses we need to slightly alter the setting by 
introducing two independent identically distributed random variables $\sigma_1$ and $\sigma_2$ and the loss function as follows 
\begin{align}
f_t(\omega_{1,t},\omega_{2,t}) =\frac{ \left(\omega_{1,t}^T\sigma_1\left(\omega_{2,t} z^t\right)-y^t \right) \left(\omega_{1,t}^T\sigma_2\left(\omega_{2,t} z^t\right)-y^t \right)}{2}.
\label{loss:non}
\end{align}
There is a crucial property of $f_t$ which is positive-homogeneity. That is, for any $c>0$, $f_t(c\omega_1,\frac{\omega_2}{c}) = f_t(\omega_1,\omega_2)$. This property allows the network to be rescaled without changing the function computed by the network.

For two-layer ReLU neural network, given $(\omega_{1,*}^p,\omega_{2,*}^p) \in \argmin_{\omega}\sum_{t=p}^{p+T}\mathop{{}\mathbb{E}_{\sigma_1,\sigma_2}} f_t(\omega_{1,t},\omega_{2,t})$, we consider regret with rolling window as
\begin{align}
\label{obj:non1}
    \max_{p\in\mathbb{N}}\min_{\substack{\left(\omega_{1,t}\right)_{t\in\mathbb{N}},\left(\omega_{2,t}\right)_{t\in\mathbb{N}}\\ \left\|\omega_{1,t}\right\|=1}} R_p(T):=\sum_{t=p}^{T+p} \mathop{{}\mathbb{E}_{\sigma_1,\sigma_2}}\left[l_t\left(\omega_{1,t},\omega_{2,t} \right)\right].
\end{align}

Next, we propose two algorithms with different learning rates or stepsizes for the two-layer neural network and analyze them with respect to (\ref{obj:non1}). 

\subsection{Algorithms}
In order to present the algorithms, let us first introduce notation and parameters. For any matrix $A$ and vector $x$, let $[A]_{ij}$ and $[x]_i$ denote the element in the $i_{\mathrm{th}}$ row and $j_{\mathrm{th}}$ column of matrix $A$ and $i_{\mathrm{th}}$ coordinate of vector $x$, respectively. Similarly, we use $[A]_{\cdot j}$ ($[A]_{i\cdot}$) to represent the $j_{\mathrm{th}}$ column ($i_{\mathrm{th}}$ row) of matrix $A$. Next, in order to be consistent, we also denote $\eta$ and $g_{1,t},g_{2,t}$ as the learning rate or stepsize and a subgradient of loss function $f_t$ associated with sample $(z^t,y^t)$, respectively. Let $\xi_1$ and $\xi_2$ be constants. 
Lastly, in order to be consistent with the notation in the convex setting, we employ $\odot$ to represent the element-wise multiplication between vectors or matrices while using standard division and square root notation for the corresponding operations element-wise in vectors and matrices. 
 
We start with {\scshape dnnGd} which is the algorithm with a fixed learning rate, Algorithm \ref{dnn GD}. 
We show later that its regret with rolling window is $\mathcal{O}(\sqrt{T})$. {\scshape dnnGD} is an analogue of the gradient descent optimization method for the online setting with the two-layer ReLU neural network, and at the same time it is an extension of {\scshape OGD}. Different from {\scshape OGD}, {\scshape dnnGD} not only modifies weights at a given iteration by following the gradient direction, but it also rescales weights based on the domain constraint in step~\ref{nonconvexSGD1}, i.e. $\omega_{1,t}$ has a fixed norm. Then, $\omega_{2,t}$ is rescaled at the same time to impose positive-homogeneity in step~\ref{nonconvexSGD2}. 

\begin{algorithm}[H]
\caption{{\scshape Deep NN Gradient Descent}}
\label{dnn GD}
\begin{algorithmic}[1]
\State $\textit{Positive parameter\,\,} \eta>0$
\For{$t=0,1,2,\cdots$}
\State $\textrm{Sample}\,\,\sigma_1,\,\sigma_2$
\State $g_{1,t}=\frac{1}{2}\left(\omega_{1,t}^T\sigma_1\left(\omega_{2,t}z^t \right)-y^t\right)\sigma_2\left(\omega_{2,t}z^t\right) +
\frac{1}{2}\left(\omega_{1,t}^T\sigma_2\left(\omega_{2,t}z^t \right)-y^t\right)\sigma_1\left(\omega_{2,t}z^t\right) $\label{dnn:g1}
\State $g_{2,t}=\frac{1}{2}\left(\omega_{1,t}^T\sigma_1\left(\omega_{2,t}z^t \right)-y^t\right)\omega_{1,t}\left(\sigma_2\left(z^t 
\right)\right)^T + \frac{1}{2}\left(\omega_{1,t}^T\sigma_2\left(\omega_{2,t}z^t \right)-y^t\right)\omega_{1,t}\left(\sigma_1\left(z^t
\right)\right)^T $\label{dnn:g2}
\State $\omega_{1,t+1} =\frac{\omega_{1,t}-\eta g_{1,t}}{\left\| \omega_{1,t}-\eta g_{1,t} \right\|\mathbin{/} \sqrt{\frac{1}{2}+\xi_1 }}$\label{nonconvexSGD1}
\State $\omega_{2,t+1} = \left(\omega_{2,t}-\eta g_{2,t}\right)\cdot \left[\left\| \omega_{1,t}-\eta g_{1,t} \right\|\mathbin{/} \sqrt{\frac{1}{2}+\xi_1} \right]$\label{nonconvexSGD2}
\EndFor
\State \textbf{end for}
\end{algorithmic}
\end{algorithm}

Taking the drawbacks of a constant learning rate into consideration, we propose
Algorithm \ref{dnn ADAM}, which is an extension of {\scshape convgAdam} for the two-layer ReLU neural network and likewise attains $\mathcal{O}(\sqrt{T})$ regret with rolling window. In {\scshape dnnAdam}, the stochastic gradients computed in steps~\ref{ddnAdam:g1} and~\ref{ddnAdam:g2} are different than those in {\scshape dnnGD}. This is due to challenges in establishing the regret bound. Nevertheless, the stochastic gradients are unbiased estimators of gradients of the loss function. An alternative is to have four samples, two per gradient group. This would also enable the regret analysis, however we only employ two of them so as to reduce the variance of the algorithm. {\scshape dnnAdam} records exponential moving average of gradients and moments in steps~\ref{dnnAdam:m1} -~\ref{dnnAdam:v2dot}. Step~\ref{dnnAdam:v2} modifies $v_{2,t}$ to be a matrix with same value in the same column. This is a divergence from standard {\scshape ADAM} which does not have this requirement. The modification is required for the regret analysis. Then, steps~\ref{dnnAdam:v1hat} and~\ref{dnnAdam:v2hat} guarantee that $\left\{\hat{v}_{1,t}\right\}$ and $\left\{\hat{v}_{2,t}\right\}$ are nondecreasing sequences element-wise. Lastly, we update weights and also perform the rescaling modification to {\scshape dnnAdam} in steps~\ref{dnnAdam:w1} and~\ref{dnnAdam:w2}. Additionally, we apply the same strategy as we mention in {\scshape convgAdam} when $[\hat{v}_{1,t}]_i = 0$ or $[\hat{v}_{2,t}]_{i,j} = 0 $. More precisely, if $[\hat{v}_{1,t}]_i=0$ ($[\hat{v}_{2,t}]_{ij} = 0 $), it implies $[g_{1,k}]_i=0$ ($[g_{2,k}]_{ij}=0$) for all $k$, which in turn yields $[m_{1,t}]_i=0$ ($[m_{2,t}]_{ij}=0$). Thus, we define $\left[m_{1,t}/\sqrt{\hat{v}_{1,t}}\right]_i = \frac{0}{0}=0$ and $\left[{m_{2,t}}/{\sqrt{\hat{v}_{2,t}}}\right]_{ij} = \frac{0}{0}=0$. Therefore, we maintain the weights from the last iteration.

\textcolor{red}{}

\begin{algorithm}[H]
\caption{{\scshape Deep NN ADAM}}
\label{dnn ADAM}
\begin{algorithmic}[1]
\State $\textit{Positive parameters\,\,} \eta, \epsilon_1, \epsilon_2, \beta_{11t}\leq 1,\beta_{12t}\leq 1,\beta_{21}\leq 1,\beta_{22}\leq 1$
\For{$t=0,1,2,\cdots$}
\State $\textrm{Sample}\,\,\sigma_1,\,\sigma_2$
\State $g_{1,t}=\left(\omega_{1,t}^T\sigma_1\left(\omega_{2,t}z^t \right)-y^t\right)\sigma_2\left(\omega_{2,t}z^t 
\right)$\label{ddnAdam:g1}
\State $g_{2,t}=\left(\omega_{1,t}^T\sigma_1\left(\omega_{2,t}z^t \right)-y^t\right)\omega_{1,t}\left(\sigma_2\left(z^t 
\right)\right)^T$\label{ddnAdam:g2}
\State $m_{1,t} = \beta_{11t}m_{1,t-1}+\left(1-\beta_{11t}\right)g_{1,t}$\label{dnnAdam:m1}
\State $m_{2,t} = \beta_{12t}m_{2,t-1}+\left(1-\beta_{12t}\right)g_{2,t}$\label{dnnAdam:m2}
\State $v_{1,t} = \beta_{21}v_{1,t-1}+\left(1-\beta_{21}\right)g_{1,t}\odot g_{1,t}$\label{dnnAdam:v1}
\State $\dot{v}_{2,t} = \beta_{22}\dot{v}_{2,t-1}+\left(1-\beta_{22}\right)g_{2,t}\odot g_{2,t}$\label{dnnAdam:v2dot}
\State $\left[v_{2,t}\right]_{ij}=\max_k \left| \left[\dot{v}_{2,t}\right]_{kj}\right|$\label{dnnAdam:v2}
\State $\hat{v}_{1t} = \max\left(v_{1t},\hat{v}_{1,t-1}\right)$\label{dnnAdam:v1hat}
\State $\hat{v}_{2t} = \max\left(v_{2t},\hat{v}_{2,t-1}\right)$\label{dnnAdam:v2hat}
\State $\omega_{1,t+1} = \frac{\omega_{1t}-\frac{\eta}{\sqrt{\hat{v}_{1t}}}\odot m_{1t}}{\left\| \omega_{1t}-\frac{\eta}{\sqrt{\hat{v}_{1t}}} \odot m_{1t}\right\|}\cdot\sqrt{\frac{\left[\frac{1}{2}+\xi_2\right]}{\left(1-\beta_{121}\right)}}$\label{dnnAdam:w1}
\State $\omega_{2,t+1} =\left(\omega_{2,t}-\frac{\eta}{\sqrt{\hat{v}_{2,t}}}\odot m_{2t}\right)\cdot\left\| \omega_{1t}-\frac{\eta}{\sqrt{\hat{v}_{1t}}}\odot m_{1t}\right\|/\sqrt{\frac{\left[\frac{1}{2}+\xi_2\right]}{\left(1-\beta_{121}\right)}} $\label{dnnAdam:w2}
\EndFor
\State \textbf{end for}
\end{algorithmic}
\end{algorithm}

\subsection{Analyses}
In this section, we discuss regret with rolling window bounds of {\scshape dnnGd} and {\scshape dnnAdam} showing that both of them attain regret with rolling window proportional to the square root of the size of the rolling window. Before establishing the regret bounds, we first require the following assumptions. 

\subsection*{Assumption 2:}
\label{as:2}

\begin{enumerate}
\item Activations $\sigma_1, \sigma_2$ are independent Bernoulli random variables with the same probability $\rho$ of success, i.e. Pr$(\sigma(x)=x) = \rho$, Pr$(\sigma(x)=0)=1-\rho$.\label{as2:1} 
\item There exists $\omega_{1,*}$ and $\omega_{2,*}$ such that $\mathop{{}\mathbb{E}}\left[\omega_{1,*}^T\sigma_1\left(\omega_{2,*}z^t \right)\right] = \rho\omega_{1,*}^T\omega_{2,*}z^t = y^t$ for all $t$. \label{as2:2} 
\item Quantities $\omega_{1,t},\omega_{2,t}$, $z^t$ and $y^t$ are all bounded for any $t$. 
In particular, let $\left\|\omega_{2,t} \right\|_F\leq \alpha$ and $\left|\left[g_{2,t}\right]_{ij}\right|\leq G_{2,\infty}$ for any $t,i,j$. \label{as2:3}
\item There exists $0 < \epsilon < \pi/2$ such that $\frac{| \left\langle
 \omega_{1,t}^T\omega_{2,t}-\omega_{1,*}^T\omega_{2,*}, z^t
 \right\rangle |}{\left\| \omega_{1,t}^T\omega_{2,t}-\omega_{1,*}^T\omega_{2,*} \right\|\left\|z^t\right\|}\geq \cos(\epsilon)$ for all $t$ when $\left(
 \omega_{1,t}^T\omega_{2,t}-\omega_{1,*}^T\omega_{2,*}\right)^T z^t \neq 0$.\label{as2:4}
\item There exits a positive constant $\mu$ such that $\mu \leq \min_{i,t,[\hat{v}_{1,t}]_i\neq 0}\left|[\hat{v}_{1,t}]_i\right|$.\label{as2:5}
\end{enumerate}

As Kawaguchi assumed in \cite{kawaguchi2016deep} and other works (\cite{dauphin2014identifying}, \cite{choromanska2015loss}, \cite{choromanska2015open}), we also assume that $\sigma$'s are Bernoulli random variables with the same probability of success and are independent from input $z^t$'s and weights $\omega$'s in \ref{as2:1} from Assumption \hyperref[as:2]{2}. \footnote{In general, the distribution of the Bernoulli random variable representing the ReLU activation function is not required to be stationary for all $t$. Since all loss functions are considered separately, we only need to assume that for every $z^t$, there is a corresponding $\rho^t$ such that $\mathrm{E}\left[\omega_{1,*}^T\omega_{2,*}z^t \right]=\rho^t\omega_{1,*}^T\omega_{2,*}z^t = y^t$, then, later in the proof, those $\rho^t$'s are absorbed into  $\mathrm{E}\left[l_t(\omega_{1,t},\omega_{2,t}|\mathrm{F}^t \right]$. Therefore, the algorithms can dynamically adapt to the new patterns in the dataset. In the proof, we simplify this process by using a constant $\rho$. Then, given $\sigma, \sigma_1, \sigma_2$ are i.i.d, $\mathrm{E}_{\sigma}\left[\left\|\omega_{1,t}^T\sigma\left(\omega_{2,t}z^t \right)-y^t \right\|^2/2 \right] = \mathrm{E}_{\sigma}\left[\left\|\omega_{1,t}^T\sigma\left(\omega_{2,t}z^t \right)-\omega_{1,*}^T\sigma\left(\omega_{2,*}z^t \right) \right\|^2/2 \right]=\frac{\rho}{2}\left(\omega_{1,t}^T\omega_{2,t}z^t-\omega_{1,*}^T\omega_{2,*}z^t \right)^2$. At the same time, the new loss function is $\mathrm{E}_{\sigma_1,\sigma_2}\left[\left(\omega_{1,t}^T\sigma_1\left(\omega_{2,t}z^t\right)-y^t \right)\left(\omega_{1,t}^T\sigma_2\left(\omega_{2,t}z^t\right)-y^t \right)/2 \right]=\mathrm{E}_{\sigma_1}\left[\omega_{1,t}^T\sigma_1\left(\omega_{2,t}z^t\right)-y^t \right]\mathrm{E}_{\sigma_2}\left[\omega_{1,t}^T\sigma_2\left(\omega_{2,t}z^t\right)-y^t \right]/2 = \frac{\rho^2}{2}\left(\omega_{1,t}^T\omega_{2,t}z^t-\omega_{1,*}^T\omega_{2,*}z^t \right)^2$. Therefore, minimizing our new loss function is the same as minimizing the original loss given that $\rho$ is a positive constant.} 
Condition in \ref{as2:2} from Assumption \hyperref[as:2]{2} states that the optimal expected loss is zero. 
This is also assumed in other prior work in offline, e.g. \cite{wu2018no}, \cite{du2017gradient}. The 3$^{rd}$ condition in Assumption \hyperref[as:2]{2} is an extension of \ref{as1:1} in Assumption \hyperref[as:1]{1}. Likewise, the constraints on $\omega_{1,t}$ and $\omega_{2,t}$ can be removed by further introducing technique discussed in Appendix \hyperref[projection]{A.1}, and consequently, $g_{1,t}$ and $g_{2,t}$ are bounded due to steps~\ref{ddnAdam:g1} and~\ref{ddnAdam:g2}. The next to the last condition in Assumption \hyperref[as:2]{2} requires that a new coming sample has to be beneficial to improve current weights. More precisely, we interpret the difference between the current weights and optimal weights as an error that needs to be corrected. Then, a new sample which is not relevant to the error vector is not allowed. In other words, we assume that the algorithm does not receive any uninformative samples. Condition \ref{as2:5} from Assumption \hyperref[as:2]{2} assumes that any nonzero $\left|[\hat{v}_{1,t}]_i\right|$ is lower bounded by a constant $\mu$ for all $t$ and $i$. It is a weak constraint since $[\hat{v}_{1,t}]_i \geq [\hat{v}_{1,t-1}]_i$ for any $t$ and $i$. In practice, we can modify it by only memorizing the first nonzero value in each coordinate and finding the smallest among these values. Otherwise, if all of $[\hat{v}_{1,t}]_i = 0$, then we can set $\mu = 1$ by default. 


The regret statement for {\scshape dnnGd} is as follows.
\begin{theorem}
\label{thm3}
If \ref{as2:1}-\ref{as2:4} in Assumption \hyperref[as:2]{2} hold, $\xi_1 = \frac{\alpha}{\cos(\epsilon)}$, and $\eta=\frac{\eta_1}{\sqrt{T}}$ for any positive constant $\eta_1$, the sequence $\omega_{1,t}$ and $\omega_{2,t}$ generated by {\scshape dnnGd} for a 2-layer ReLU neural network achieves $\max_{p\in\mathbb{N}}\mathop{{}\mathbb{E}}\left[R_p(T)\right]\leq \mathcal{O}(\sqrt{T})$.
\end{theorem}
The proof is in Appendix \hyperref[proof:thm3]{D}.
Based on the fact that the loss function is nonconvex, i.e., we no longer have a direct relationship between the loss function error $f_t(\omega_{1,t},\omega_{2,t})-f_t(\omega_{1,*},\omega_{2,*})$ and  $\left\langle \omega_{1,t}^T\omega_{2,t}-\omega_{1,*}^T\omega_{2,*},\omega_{1,t}^T\bigtriangledown_{\omega_2} f_t(\omega_{1,t},\omega_{2,t})\right.$ $\left.+\left(\bigtriangledown_{\omega_1} f_t(\omega_{1,t},\omega_{2,t})\right)^T\omega_{2,t}\right\rangle$, any technique that relies on the property of convexity is inappropriate. The main challenges are coming from building a bridge between the loss function error $f_t(\omega_{1,t}^T\omega_{2,t})-f_t(\omega_{1,*}\omega_{2,*})$ and the weight error $\omega_{1,t}^T\omega_{2,t}-\omega_{1,*}^T\omega_{2,*}$. To address this problem, we explore the difference between $\omega_{1,t+1}^T\omega_{2,t+1}-\omega_{1,*}^T\omega_{2,*}$ and $\omega_{1,t}^T\omega_{2,t}-\omega_{1,*}^T\omega_{2,*}$ in detail. 

The steps to study the regret with rolling window are as follows. Based on steps~\ref{dnn:g1} -~\ref{nonconvexSGD2}, we expand $\omega_{1,t+1}^T\omega_{2,t+1}-\omega_{1,*}^T\omega_{2,*}$ to establish a relationship for the sequence of the weight error $\omega_{1,t}^T\omega_{2,t}-\omega_{1,*}^T\omega_{2,*}$. In association with explicit formulas of gradients and condition \ref{as2:4} in Assumption \hyperref[as:2]{2}, we obtain the loss function error $f_t(\omega_{1,t},\omega_{2,t})-f_t(\omega_{1,*},\omega_{2,*})$. Meanwhile, all of the remaining terms are bounded due to condition \ref{as2:3} from Assumption \hyperref[as:2]{2}. Combined with the fact that $\omega_{1,t}$ has a constant norm, the regret with rolling window bound of {\scshape dnnGd} is achieved by applying the law of iterated expectation. 

At the same time, our proof is flexible enough to extend to standard online setting. For a constant learning rate, Appendix \hyperref[proof:thm3]{D} provides the necessary details for the standard case. In summary, regret of $\mathcal{O}(\sqrt{T})$ is achieved. We note that such a result has only been known for the diminishing learning rate and thus we extend the prior knowledge by covering the constant learning rate case. 

The adaptive learning setting algorithm {\scshape dnnAdam} has the same regret bound as stated in the following theorem.
\begin{theorem}
\label{thm4}
If Assumption \hyperref[as:2]{2} holds, $\eta=\frac{\eta_1}{\sqrt{T}}$ for any positive constant $\eta_1$, $\beta_{111}, \beta_{121}, \beta_{21},\beta_{22}$ are constants between $0$ and $1$ such that $\lambda_1:=\frac{\beta_{111}}{\beta_{21}}\leq 1$ and $\lambda_2:=\frac{\beta_{121}}{\beta_{22}}\leq 1$, $\beta_{11t}=\beta_{111}\gamma_1^t$ and $\beta_{12t}=\beta_{121}\gamma_2^t$ with $0<\gamma_1,\gamma_2<1$, and $\xi_2 = \frac{\alpha G_{2,\infty}}{\mu\cos{(\epsilon)}}$, then, the sequence $\omega_{1,t}$ and $\omega_{2,t}$ generated by {\scshape dnnAdam} for the 2-layer ReLU neural network achieves $\max_{p\in\mathbb{N}}\mathop{{}\mathbb{E}}\left[R_p(T)\right]\leq \mathcal{O}(\sqrt{T})$.
\end{theorem}

The proof is in  Appendix \hyperref[proof:thm4]{E}.
Similar to the difficulty faced in the proof of Theorem \ref{thm3}, we do not possess a relationship between the loss function error $f_t(\omega_{1,t},\omega_{2,t})-f_t(\omega_{1,*},\omega_{2,*})$ and $\left\langle \omega_{1,t}^T\omega_{2,t}-\omega_{1,*}^T\omega_{2,*},\omega_{1,t}^T\bigtriangledown_{\omega_2} f_t(\omega_{1,t},\omega_{2,t})+\left(\bigtriangledown_{\omega_1} f_t(\omega_{1,t},\omega_{2,t})\right)^T\omega_{2,t}\right\rangle$. Even worse, the variance of the algorithm caused by merging all previous information and normalizing the stepsize makes the relationship between the loss function error $f_t(\omega_{1,t},\omega_{2,t})-f_t(\omega_{1,*},\omega_{2,*})$ and the weight error $\omega_{1,t}^T\omega_{2,t}-\omega_{1,*}^T\omega_{2,*}$ more ambiguous. The way we deal with this is by treating $\frac{m_t}{\sqrt{v}_t}$ together as the gradient first and then extracting the effective gradient out from it and bounding the remaining terms. 

The structure of the technical proof is similar to that of Theorem \ref{thm3}. We first establish a relationship for the sequence of the weight error $\omega_{1,t}^T\omega_{2,t}-\omega_{1,*}^T\omega_{2,*}$ by multiplying $\sqrt[4]{\hat{v}_{2,t}}$. Then, using the definitions of $\beta$'s, $\lambda$'s and $\gamma$'s, we bound all the terms without the stepsize by constants except those which potentially can contribute to the loss function. To this end, we obtain a relationship between the weight error $\omega_{1,t}^T\omega_{2,t}-\omega_{1,*}^T\omega_{2,*}$ and the loss function. Finally, combined with step\ref{dnnAdam:v2hat} and the law of iterated expectation, we are able to argue $\mathcal{O}(\sqrt{T})$ regret with rolling window for {\scshape dnnAdam}.

Likewise, we are able to extend the proof of Theorem \ref{thm4} to the standard online setting for {\scshape dnnAdam}. We do not need to make any change to establish $\mathcal{O}(\sqrt{T})$. For diminishing stepsize $\mu$, a slight change to the proof is indeed. Details are provided in Appendix \hyperref[projection]{A.2}.  


\section{Numerical Study}
In this section, we compare the {\scshape convgAdam} method with {\scshape OGD} \cite{zinkevich2003online} for solving problem (\ref{goal}) with a long sequence of data points (mimicking streaming). We conduct experiments on the MNIST8M dataset and two other different-size real datasets from the Yahoo! Research Alliance Webscope program. For all of these datasets, we train  multi-class hinge loss support vector machines (SVM) \cite{shalev2014understanding} and we assume that the samples are streamed one by one based on a certain random order.  
For all the figures provided in this section, the horizontal axis is in $10^5$ scale. Moreover, we set $\beta_1=0.8$ and $\beta_2=0.81$ in  {\scshape convgAdam}. We mostly capture the log of the loss function value which is defined as $\max_{p\in\mathbb{N}}\min_{\left(\omega_{t}\right)_{t\in\mathbb{N}}} \sum_{t=p}^{T+p} f_t(\omega_t)$.


\subsection{Multiclass SVM with Yahoo! Targeting User Modeling Dataset}
We first compare {\scshape convgAdam} with {\scshape OGD} using the Yahoo! user targeting and interest prediction dataset consisting of Yahoo user profiles\footnote{https://webscope.sandbox.yahoo.com/catalog.php?datatype=a}. It contains 1,589,113 samples (i.e., user profiles), represented by a total of 13,346 features and 380 different classification problems (called labels in the supporting documentation) each one with 3 classes.

First, we pick the first label out and conduct a sequence of experiments with respect to this label. 
The most important results are presented in Figure \hyperref[fig:1]{1} for {\scshape OGD} and Figure \hyperref[fig:2]{2} for {\scshape convgAdam}. In Figures \hyperref[fig1]{1(a)} and \hyperref[fig4]{2(a)}, we consider the cases when the learning rate or step size varies from $0.1$ to $5\cdot 10^{-6}$ while keeping the order and $T$ fixed at 1,000. Figures \hyperref[fig2]{1(b)} and \hyperref[fig5]{2(b)} provide the influence of the order of the sequence. Figures \hyperref[fig3]{1(c)} and \hyperref[fig6]{2(c)} represent the case where $T$ varies from $10$ to $10^5$ with a fixed learning rate or step size. 
Lastly, in Figure \hyperref[fig7]{2(d)}, we compare the performance of {\scshape convgAdam} and {\scshape OGD} with certain learning rates and step sizes.

In these plots, we observe that {\scshape convgAdam} outperforms {\scshape OGD} for most of the learning rates and step sizes, and definitely for promissing choices. More precisely, in Figure \hyperref[fig1]{1(a)} and \hyperref[fig4]{2(a)}, we discover that 0.1/1000 and 3/$\sqrt{1000}$ are two high-quality learning rate and stepsize values which have relatively low error and are learning for {\scshape OGD} and {\scshape convgAdam}, respectively. Therefore, we apply those two learning rates for the remaining experiments on this dataset. In Figures \hyperref[fig2]{1(b)} and \hyperref[fig5]{2(b)}, we observe that the perturbation caused by the change of the order is negligible especially when compared to the loss value, which is a positive characteristic. Thus, in the remaining experiments, we no longer need to consider the impact of the order of the sequence. From Figure \hyperref[fig3]{1(c)} and Figure \hyperref[fig6]{1(d)}, we discover that the loss and $T$ have a significantly positive correlation as we expect. Notice that changing $T$ but fixing the learning rate or stepsize essentially means containing more samples in the regret, in other words, the regret for $T=100$ is roughly $10$ times the regret for $T=10$. Since the pattern in the figures is preserved for the different $T$ values for {\scshape OGD} and {\scshape convgAdam}, in the remaining experiments we fix $T$. In Figure \hyperref[fig6]{2(c)}, we discover that too big $T$ or too small $T$ causes poor performance and therefore, for the remaining experiments, we set $T=1,000$ whenever $T$ is fixed. From Figure \hyperref[fig7]{2(d)}, we observe that {\scshape convgAdam} outperforms {\scshape OGD}.

\begin{figure}[H]
\minipage{0.33\textwidth}
\includegraphics[width=\linewidth]{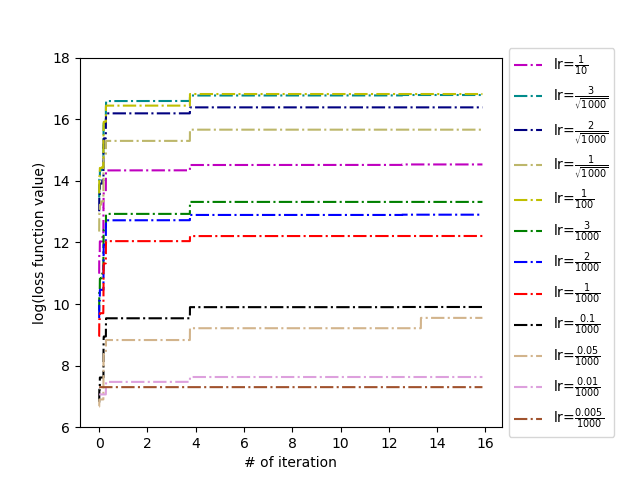}
\caption*{(a)}
\label{fig1}
\endminipage\hfill
\minipage{0.33\textwidth}
\includegraphics[width=\linewidth]{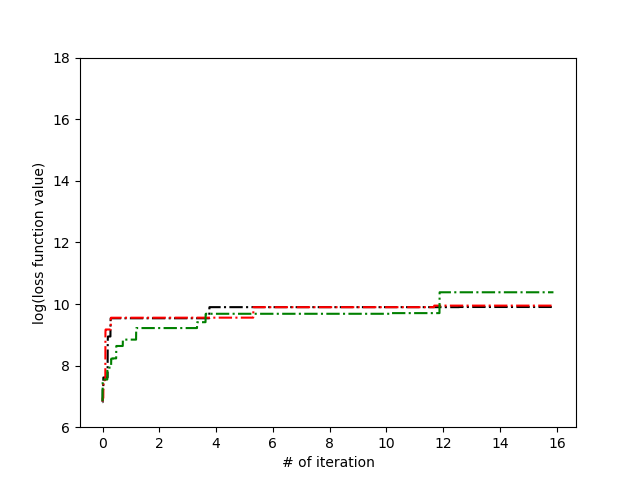}
\caption*{(b)}
\label{fig2}
\endminipage\hfill
\minipage{0.33\textwidth}
\includegraphics[width=\linewidth]{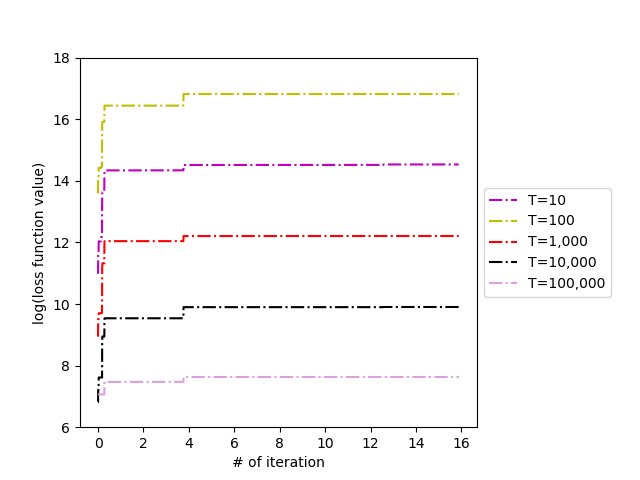}
\caption*{(c)}
  \label{fig3}
\endminipage\hfill
\caption{Comparison of {\scshape OGD} for different orders, learning rates and $T$}
\label{fig:1}
\end{figure}

\begin{figure}[H]
\minipage{0.5\textwidth}
\includegraphics[width=\linewidth]{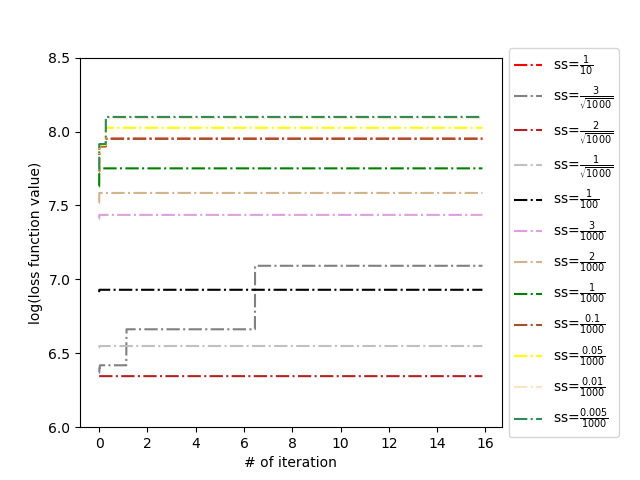}
\caption*{(a)}
\label{fig4}
\endminipage\hfill
\minipage{0.5\textwidth}
\includegraphics[width=\linewidth]{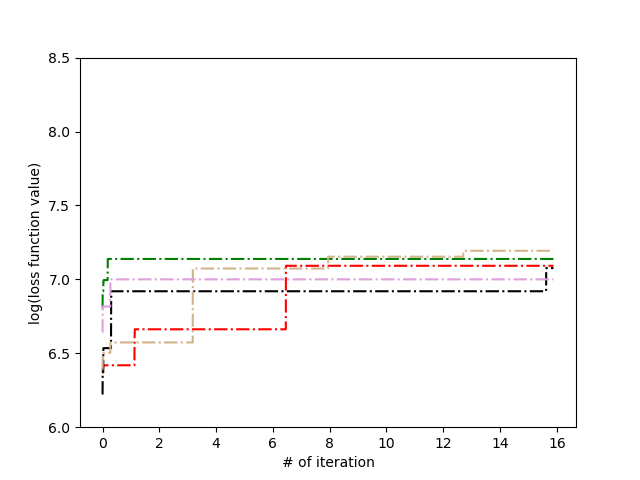}
\caption*{(b)}
  \label{fig5}
\endminipage\hfill
\minipage{0.5\textwidth}
\includegraphics[width=\linewidth]{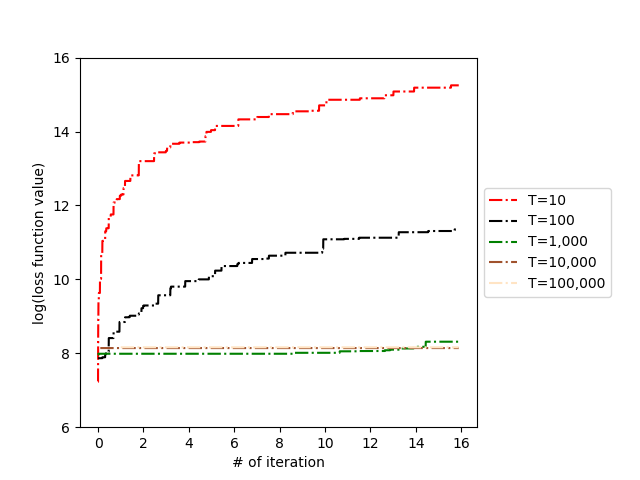}
\caption*{(c)}
\label{fig6}
\endminipage\hfill
\minipage{0.5\textwidth}
\includegraphics[width=\linewidth]{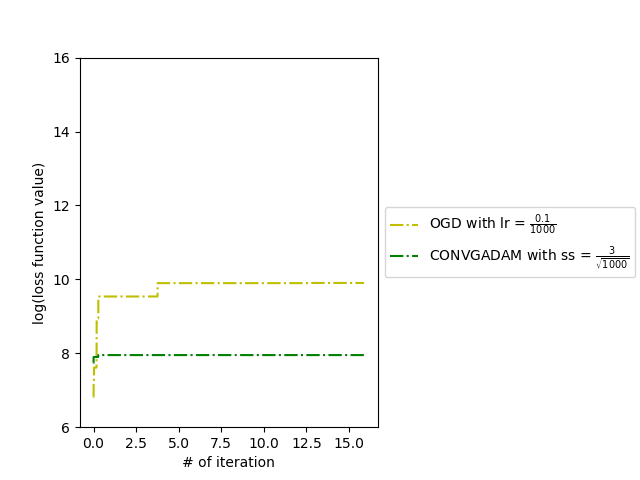}
\caption*{(d)}
\label{fig7}
\endminipage\hfill
\caption{Comparison of {\scshape convgAdam} for different orders, stepsizes and $T$}
\label{fig:2}
\end{figure}

After studying the algorithms on the first label, we test them on the next four labels. In Figure \hyperref[fig:3]{3}, we compare the performance of {\scshape convgAdam} for different $T$ and the difference with {\scshape OGD} on the four labels. In these plots, we observe that $T=1000$ provides a more stable and better performance than the other two values. Moreover, {\scshape convgAdam} outperforms {\scshape OGD} for all considered learning rates and step sizes.

\begin{figure}[H]
\minipage{0.5\textwidth}
\includegraphics[width=\linewidth]{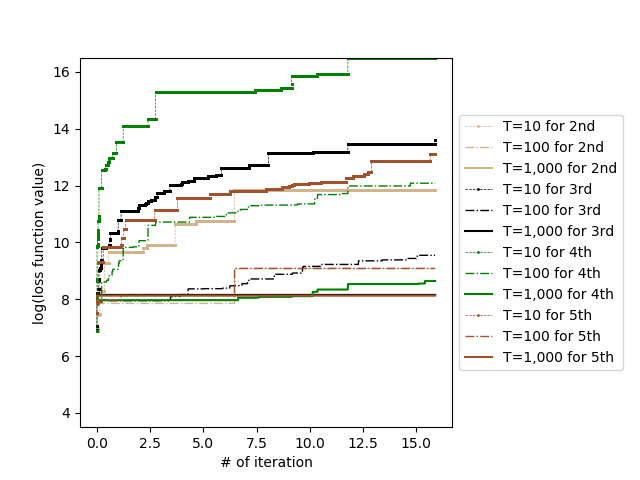}
\caption*{(a) Comparison of {\scshape convgAdam} for different $T$}
\label{fig8}
\endminipage\hfill
\minipage{0.5\textwidth}
\includegraphics[width=\linewidth]{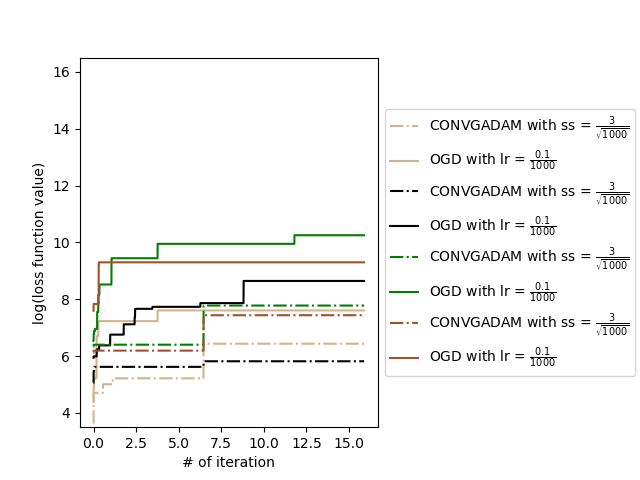}
\caption*{(b) Comparison of {\scshape OGD} and {\scshape convgAdam}}
\label{fig9}
\endminipage\hfill
\caption{Performance of {\scshape convgAdam} and {\scshape OGD} on the remaining labels}
\label{fig:3}
\end{figure}

\subsection{Multiclass SVM with Yahoo! Learn to Rank Challenge Dataset}
In this set of experiments, we study the performances of {\scshape convgAdam} and {\scshape OGD} on Yahoo! Learn to Rank Challenge Dataset\footnote{https://webscope.sandbox.yahoo.com/catalog.php?datatype=c}. The dataset contains 473,134 samples, represented by a total of 700 features and 5 classes.

Figures \hyperref[fig10]{4(a)} and \hyperref[fig11]{4(b)} show the performances of {\scshape OGD} and {\scshape convgAdam} for different learning rates and stepsizes. Figure \hyperref[fig12]{4(c)} provides the performance of {\scshape convgAdam} for different $T$. Lastly, Figure \hyperref[fig13]{4(d)} compares the performance of {\scshape convgAdam} and {\scshape OGD} for a set of good learning rates but same $T$. 

From Figures \hyperref[fig10]{4(a)} and \hyperref[fig11]{4(b)}, we select the learning rate and stepsize 3/${\sqrt{1000}}$ and 2/${\sqrt{1000}}$ for {\scshape convgAdam} and {\scshape OGD}, respectively. From Figure \hyperlink{fig13}{4(d)}, we discover the superior behavior of {\scshape convgAdam} over {\scshape OGD} as we expect.

\begin{figure}[H]
\minipage{0.5\textwidth}
\includegraphics[width=\linewidth]{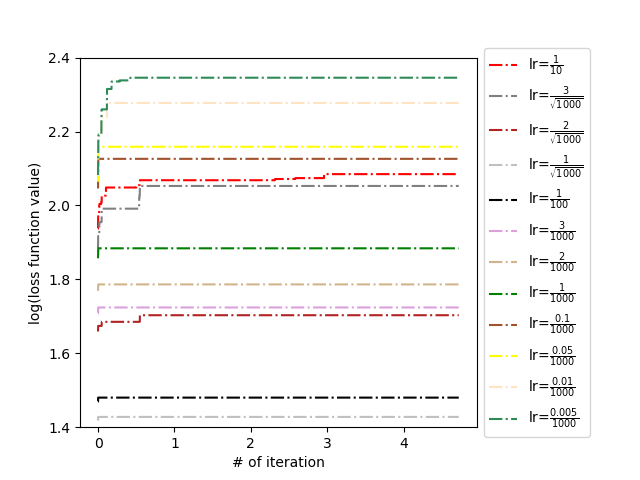}
\caption*{(a) Comparison of {\scshape OGD} for different learning rates}
\label{fig10}
\endminipage\hfill
\minipage{0.5\textwidth}
\includegraphics[width=\linewidth]{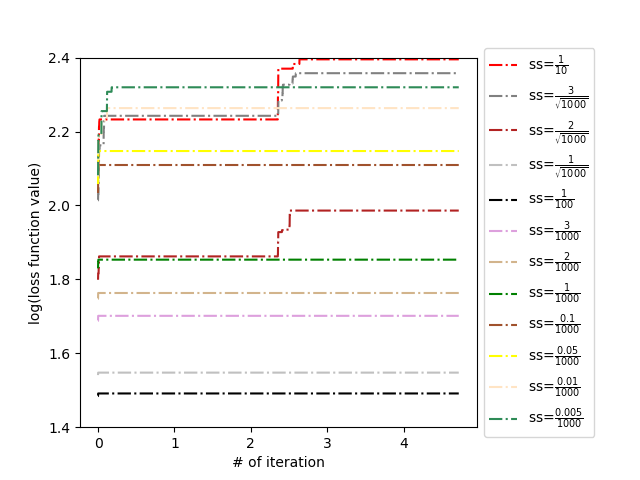}
\caption*{(b) Comparison of {\scshape convgAdam} for different stepsizes}
\label{fig11}
\endminipage\hfill
\minipage{0.5\textwidth}
\includegraphics[width=\linewidth]{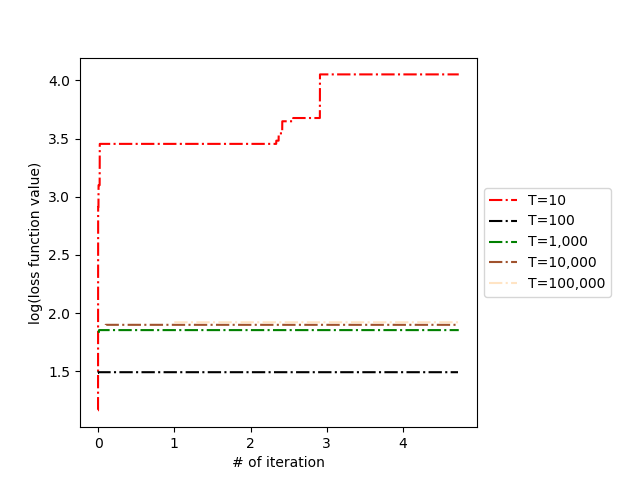}
\caption*{(c) Comparison of {\scshape convgAdam} for different $T$}
\label{fig12}
\endminipage\hfill
\minipage{0.5\textwidth}
\includegraphics[width=\linewidth]{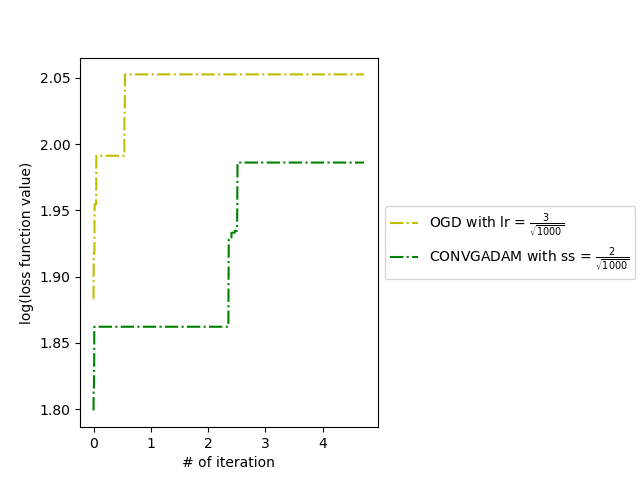}
\caption*{(d) Comparison of {\scshape convgAdam} and {\scshape OGD}}
\label{fig13}
\endminipage\hfill
\label{fig:4}
\caption{Performance of {\scshape convgAdam} on Learn to Rank Challenge dataset}
\end{figure}

\subsection{Multiclass SVM with MNIST8M Dataset}
In this set of experiments, we study the performances of {\scshape convgAdam} and {\scshape OGD} on MNIST8M Dataset\footnote{https://www.csie.ntu.edu.tw/~cjlin/libsvmtools/datasets/}. The dataset is generated on the fly by performing careful elastic deformation of the original MNIST training set. The dataset contains 8,100,000 samples, represented by a total of 784 features and 10 classes.

In Figures \hyperref[fig14]{5(a)} and \hyperref[fig15]{5(b)}, we compare the performances of {\scshape OGD} and {\scshape convgAdam} for different learning rates and stepsizes. Figure \hyperref[fig16]{5(c)} shows that performance of {\scshape convgAdam} for different $T$. Lastly, Figure \hyperref[fig17]{5(d)} depicts the comparison of {\scshape convgAdam} and {\scshape OGD}. From Figures \hyperref[fig14]{5(a)} and \hyperref[fig15]{5(b)}, we select the stepsize 2/${\sqrt{1000}}$ and the learning rate of 1/${1000}$. As we observe, {\scshape convgAdam} always exhibits a better performance than {\scshape OGD}.

\begin{figure}[H]
\minipage{0.5\textwidth}
\includegraphics[width=\linewidth]{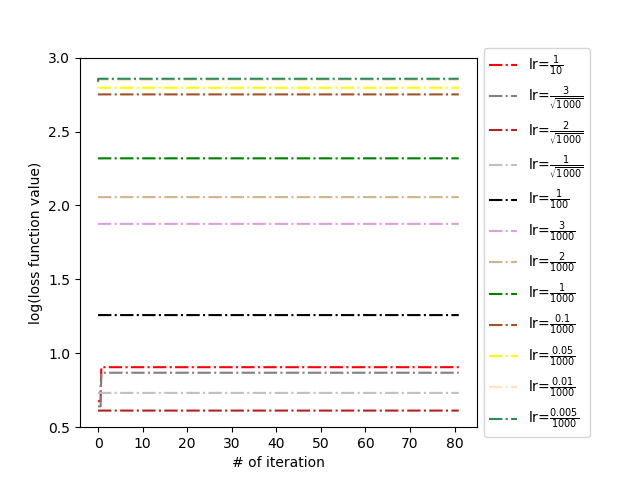}
\caption*{(a) Comparison of {\scshape OGD} for different learning rates}
\label{fig14}
\endminipage\hfill
\minipage{0.5\textwidth}
\includegraphics[width=\linewidth]{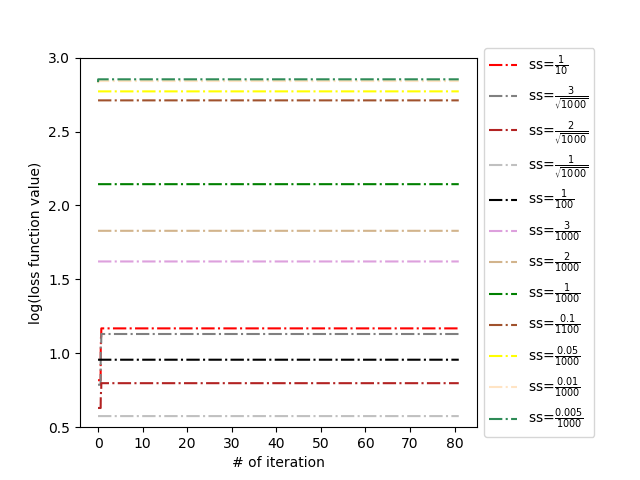}
\caption*{(b) Comparison of {\scshape convgAdam} for different stepsizes}
\label{fig15}
\endminipage\hfill
\minipage{0.5\textwidth}
\includegraphics[width=\linewidth]{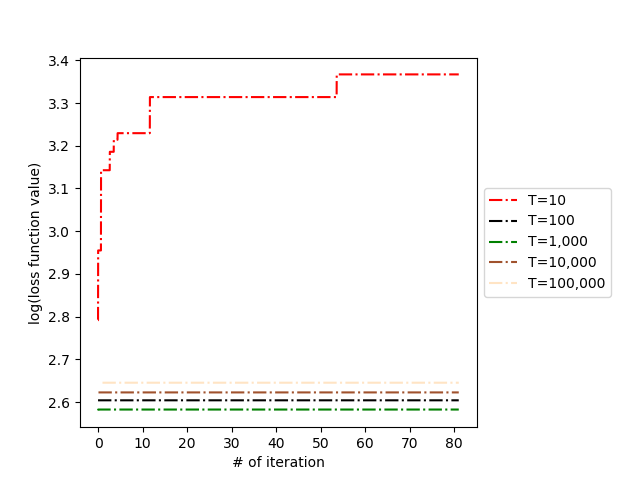}
\caption*{(c) Comparison of {\scshape convgAdam} for different $T$}
\label{fig16}
\endminipage\hfill
\minipage{0.5\textwidth}
\includegraphics[width=\linewidth]{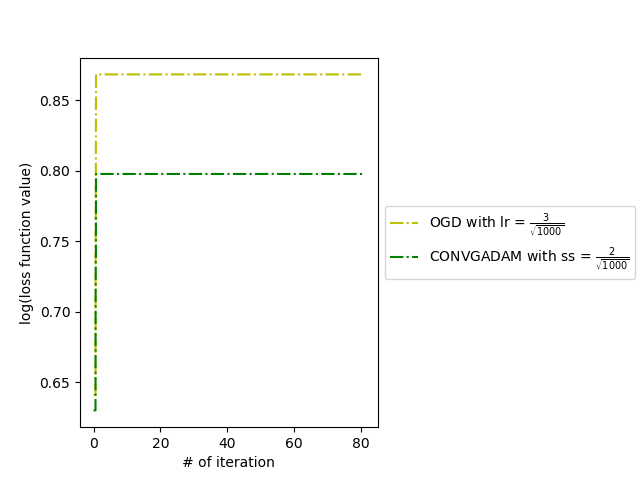}
\caption*{(d) Comparison of {\scshape convgAdam} and {\scshape OGD}}
\label{fig17}
\endminipage\hfill
\caption{Performance of {\scshape convgAdam} on MINST8M dataset}
\label{fig:5}
\end{figure}

\newpage
\bibliographystyle{apa}
\bibliography{references}

\newpage
\section{Appendix}
\subsection*{A\tab Extensions}
We first introduce techniques to guarantee boundedness of the weight $\omega$, i.e. how to remove condition \ref{as1:1} in Assumption \hyperref[as:1]{1}. We then point out problems in the proofs of {\scshape AMSGrad} \cite{reddi2018convergence} and {\scshape AdaBound} \cite{luo2019adaptive} and provide a different proof for {\scshape AMSGrad}.
\subsubsection*{A.1\tab Unbounded Case}
\label{projection}
Projection is a popular technique to guarantee that a weight does not exceed a certain bound (\cite{blum1998line}, \cite{Hazan:2007:LRA:1296038.1296051},  \cite{duchi2011adaptive}, \cite{luo2019adaptive}). For unbounded weight $\hat{\omega}$, we introduce the following notation. Given convex sets $\mathcal{P}_1$, $\mathcal{P}_2$, vectors $\omega_1,\omega_1', g_1$ and matrix $\hat{v}$, we define projections
\begin{align*}
&\Pi_{\mathcal{P}_1}(\hat{\omega})=\argmin_{\omega\in\mathcal{P}_1}\left\|\omega-\hat{\omega}\right\|\\
&\Pi^1_{\mathcal{P}_1,\mathcal{P}_2,\omega_1,g_1,\omega_1'}(\hat{\omega}_2)\\
&=
\argmin_{\omega_{2}': \omega_2'\cdot\left[\left\| \omega_{1}'-\eta g_{1} \right\|\mathbin{/} \sqrt{\frac{1}{2}+\xi_1 } \right] \in\mathcal{P}_2}\left\|\omega_2'-
\argmin_{\omega_2: \omega_1^T\omega_2\in\mathcal{P}_1}\left\|\omega_1^T\omega_2-\omega_1^T\hat{\omega}_2\right\|\right\|\\
&\Pi^2_{\mathcal{P}_1,\mathcal{P}_2,\omega_1,g_1,\omega_1', \hat{v}}(\hat{\omega}_2)\\
&=\argmin_{\omega_2': \omega_2'\cdot \left[\left\| \omega_{1}'-\eta g_{1} \right\|\mathbin{/} \sqrt{\frac{1}{2}+\xi_2 } \right]\in\mathcal{P}_2}\left\|\omega_2'-\argmin_{\omega_2: \omega_1^T\omega_2\in\mathcal{P}_1}\left\|\left(\sqrt[4]{\hat{v}}\odot\omega_2\right)^T\omega_1-\left(\sqrt[4]{\hat{v}}\odot\hat{\omega}_2\right)^T\omega_1\right\|\right\|.
\end{align*}
Projection $\Pi$ is the standard projection which maps vector $\hat{\omega}$ into set $\mathcal{P}_1$. If an optimal weight $\omega_*$ is such that $\omega_*\in\mathcal{P}_1$, then we have
\begin{align*}
    \left\|\Pi_{\mathcal{P}_1}(\hat{\omega}_{t+1})-\omega_*\right\|\leq \left\|\hat{\omega}_{t+1}-\omega_*\right\|,
\end{align*}
which could be directly applied in the proofs of Theorem \ref{thm1} and \ref{thm2}.

For $\Pi^1$ and $\Pi^2$, we could regard them as a combination of two standard projections. Note that, for the outer projection, we require that it does not affect the product of $\omega_1^T\omega_2$, which could be done by projection methods for linear equality constraints. In this way, we have
\begin{align*}
    &\left\|\omega_{1,t+1}^T\Pi^1_{\mathcal{P}_1,\mathcal{P}_2,\omega_{1,t+1},g_{1,t},\omega_{1,t}}(\hat{\omega}_{2,t+1})-\omega_{1,*}^T\omega_{2,*}\right\|\leq \left\|\omega_{1,t+1}^T\hat{\omega}_{2,t+1}-\omega_{1,*}^T\omega_{2,*}\right\|\\
    &\left\|\left(\sqrt[4]{\hat{v}_{2,t}}\odot \Pi^2_{\mathcal{P}_1,\mathcal{P}_2,\omega_{1,t+1},g_{1,t},\omega_{1,t},\hat{v}_{2,t}}(\hat{\omega}_{2,t+1})\right)^T\omega_{1,t+1}-\left(\sqrt[4]{\hat{v}_{2,t}}\odot\omega_{2,*}\right)^T\omega_{1,*}\right\|\\
    &\quad\quad\leq \left\|\left(\sqrt[4]{\hat{v}_{2,t}}\odot\left(\hat{\omega}_{2,t+1}\right)\right)^T\omega_{1,t+1}-\left(\sqrt[4]{\hat{v}_{2,t}}\odot\omega_{2,*}\right)^T\omega_{1,*}\right\|,
\end{align*}
which could also be directly applied in the proofs of Theorem \ref{thm3} and \ref{thm4}.
\subsubsection*{A.2\tab Standard setting of {\scshape Adam}}
\label{correction}
First, let us point out the problem in {\scshape AMSGrad} \cite{reddi2018convergence}. At the bottom of Page 18 in \cite{reddi2018convergence}, the authors obtain an upper bound for the regret which has a term containing $\sum_{t=1}^T\frac{\beta_{1t}\hat{v}_{t,i}^{1/2}}{\alpha_t}$. Without assuming that $\beta_{1t}$ is exponentially decaying, it is questionable to establish $\mathcal{O}(\sqrt{T})$ given $\alpha_t=\frac{1}{\sqrt{t}}$ since $\sum_{t=1}^T\sqrt{t} > \mathcal{O}(\sqrt{T})$. Although this questionable term can be bounded by assumptions on $\beta_{1t}$, the last term in Theorem 4 is $\mathcal{O}(\log(T)\sum_{i=1}^d\left\|g_{1:T,i} \right\|_2)=\mathcal{O}(\log(T)\sqrt{T})$ since $g_{1:T,i}$ is the concatenation of the gradients from 0 to current time $T$ in the $i^{th}$ coordinate. Moreover, 
the authors argue that decaying $\beta_{1t}$ is crucial to guarantee the convergence, however, our proof shows $\mathcal{O}(\sqrt{T})$ regret for {\scshape AMSGrad} with constant $\beta$ and both constant and diminishing stepsizes, which is more practically relevant. For a diminishing stepsize, the slight change we need to make in the proof is that $\eta_t$ needs to be considered together with $\sqrt{\hat{v}_{t,j}}$ in (\ref{ADAMeq:8}) and the rest of proof of Theorem \ref{thm2}. Applying the fact that $\frac{\sqrt{\hat{v}_{t,j}}}{\eta_t}\geq \frac{\sqrt{\hat{v}_{t-1,j}}}{\eta_t}$ and $\sum_{t=1}^T\frac{1}{\sqrt{t}}=2\sqrt{T}-1$ yields $\mathcal{O}(\sqrt{T})$ regret in standard online setting.


Table \ref{resultsummary} summarizes the various regret bounds in different convex settings. 

\begin{table}[H]
\begin{tabular}{|l|l|l|l|l|}
\hline
\multirow{2}{*}{} & \multicolumn{2}{c|}{gradient descent}& \multicolumn{2}{c|}{Adam}\\ \cline{2-5} 
& constant & diminishing & constant & diminishing\\ \hline
standard online   & \begin{tabular}[c]{@{}l@{}}$\mathcal{O}(\sqrt{T})$(us)\\ $\mathcal{O}(T)$\cite{zinkevich2003online}\end{tabular} & \begin{tabular}[c]{@{}l@{}}$\mathcal{O}(\sqrt{T})$(us)\\ $\mathcal{O}(\sqrt{T})$\cite{zinkevich2003online}\end{tabular}& $\mathcal{O}(\sqrt{T})$(us)& \begin{tabular}[c]{@{}l@{}}$\mathcal{O}(\sqrt{T})$(us) \\ $\mathcal{O}(\sqrt{T})$\cite{reddi2018convergence} (flawed)\\ $\mathcal{O}(\log(T)\sqrt{T})$\cite{reddi2018convergence} (true)\end{tabular} \\ \hline
streaming& $\mathcal{O}(\sqrt{T})$(us)& $\mathcal{O}(T)$(us)&$\mathcal{O}(\sqrt{T})$(us)&$\mathcal{O}(T)$(us)\\ \hline
\end{tabular}
\captionsetup{justification=centering}
\captionsetup{width=.75\textwidth}
\caption{Summary of known regret bounds for online learning and streaming in convex setting}
\label{resultsummary}
\end{table}

\subsection*{B\tab Regret with Rolling Window Analysis of {\scshape OGD}}
\subsubsection*{Proof of Theorem \ref{thm1}}
\label{proof:thm1}
\begin{proof}

For any $p\in\mathbb{N}$ and fixed $T$, from step~\ref{SGDupdate} in Algorithm~\hyperref[SGD]{1}, for any $\omega^*$, we obtain
\begin{align*}
&\left\|\omega_{t+1}-\omega^*\right\|^2=\left\|\omega_t-\eta \bigtriangledown f_t(\omega_t)-\omega^*\right\|^2\\
=&\left\|\omega_t-\omega^*\right\|^2-2\eta\left\langle \omega_t-\omega^*,\bigtriangledown f_t(\omega_t)\right\rangle+\eta^2 \left\|\bigtriangledown f_t(\omega_t)\right\|^2,
\end{align*}
which in turn yields
\begin{align}
\label{SGDeq:1}
\left\langle \omega_t-\omega^*,\bigtriangledown f_t(\omega_t)\right\rangle = \frac{\left\|\omega_t-\omega^*\right\|^2-\left\|\omega_{t+1}-\omega^*\right\|^2}{2\eta}+\frac{\eta}{2} \left\|\bigtriangledown f_t(\omega_t)\right\|^2.
\end{align}
Applying convexity of $f_t$ yields
\begin{align}
f_t(\omega_t)-f_t(\omega^*)\leq
\left\langle \omega_t-\omega^*,\bigtriangledown f_t(\omega_t)\right\rangle.
\label{SGDeq:2}
\end{align}
Inserting (\ref{SGDeq:1}) into (\ref{SGDeq:2}) gives
\begin{align*}
f_t(\omega_t)-f_t(\omega^*)\leq \frac{\left\|\omega_t-\omega^*\right\|^2-\left\|\omega_{t+1}-\omega^*\right\|^2}{2\eta}+\frac{\eta}{2} \left\|\bigtriangledown f_t(\omega_t)\right\|^2.
\end{align*}
By summing up all differences, we obtain
\begin{align}
\label{SGDeq:3}
\sum_{t=p}^{T+p}\left[f_t(\omega_t)-f_t(\omega^*)\right]&\leq\frac{1}{2}\sum_{t=p}^{T+p}\left[\frac{\left\|\omega_t-\omega^*\right\|^2-\left\|\omega_{t+1}-\omega^*\right\|^2}{\eta}+\eta \left\|\bigtriangledown f_t(\omega_t)\right\|^2\right]\nonumber\\
&\leq \frac{1}{2}\left(\frac{\left\|\omega_p-\omega^*\right\|^2}{\eta} \right)+dG_{\infty}\sum_{t=p}^{T+p}\eta\nonumber\\
&\leq \frac{D_{\infty}^2\sqrt{T}}{2\eta_1}+dG_{\infty}\eta_1\sqrt{T}=\mathcal{O}(\sqrt{T}).
\end{align}
The second inequality holds due to \ref{as1:2} in Assumption \hyperref[as:1]{1} and the last inequality uses \ref{as1:4} in Assumption \hyperref[as:1]{1} and the definition of $\eta$. Since (\ref{SGDeq:3}) holds for any $p$ and $\omega^*$, setting $\omega^*=\omega_p^*$ for each $p$ yields the statement in Theorem \ref{thm1}. 
\end{proof}
\subsection*{C\tab Regret with Rolling Window Analyses of {\scshape convgAdam}}
\begin{lem}
\label{lem1}
Under the conditions assumed in Theorem \hyperref[thm2]{2}, we have
\begin{align*}
\sum_{t=p}^{T+p}\left\|\frac{1}{\sqrt[4]{\hat{v}_t}}\odot m_t\right\|^2\leq \mathcal{O}(T).
\end{align*}
\end{lem}
\begin{proof}[Proof of Lemma \ref{lem1}]
By the definition of $\hat{v}_t$, for any $t=p,p+1,\cdots, T+p$, we obtain
\begin{align}
\label{ADAMeq:1}
\left\|\frac{1}{\sqrt[4]{\hat{v}_t}}\odot m_t\right\|^2&=\sum_{j=1}^d\frac{m_{t,j}^2}{\sqrt{\hat{v}_{t,j}}}\leq \sum_{j=1}^d\frac{m_{t,j}^2}{\sqrt{v_{t,j}}}=\sum_{j=1}^d\frac{\left( (1-\beta_1)\sum_{i=1}^t\beta_1^{t-i}g_{i,j}\right)^2}{\sqrt{(1-\beta_2)\sum_{i=1}^t\beta_2^{t-i}g_{i,j}^2}}\nonumber\\
&\leq \frac{(1-\beta_1)^2}{\sqrt{1-\beta_2}} \sum_{j=1}^d\frac{\left(\sum_{i=1}^t\beta_1^{t-i} \right)\left(\sum_{i=1}^t\beta_1^{t-i}g_{i,j}^2 \right)}{\sqrt{\sum_{i=1}^t\beta_2^{t-i}g_{i,j}^2}}\nonumber\\
&\leq \frac{1-\beta_1}{\sqrt{1-\beta_2}} \sum_{j=1}^d\frac{\sum_{i=1}^t\beta_1^{t-i}g_{i,j}^2 }{\sqrt{\sum_{i=1}^t\beta_2^{t-i}g_{i,j}^2}}\leq \frac{1-\beta_1}{\sqrt{1-\beta_2}}\sum_{j=1}^d\sum_{i=1}^t\left(\frac{\beta_1}{\sqrt{\beta_2}} \right)^{t-i}\left\|g_{i,j}\right\|_2\nonumber\\
&=\frac{1-\beta_1}{\sqrt{1-\beta_2}}\sum_{j=1}^d\sum_{i=1}^t\lambda^{t-i}\left\|g_{i,j}\right\|_2.
\end{align}
The second equality follows from the updating rule of Algorithm \ref{convgADAM}. The second inequality follows from the Cauchy-Schwarz inequality, while the third inequality follows from the inequality $\sum_{i=1}^t\beta_1^{t-i}\leq \frac{1}{1-\beta_1}$. Using (\ref{ADAMeq:1}) for all time steps yields
\begin{align}
\label{ADAMeq:2}
&\sum_{t=p}^{T+p}\frac{1}{\sqrt{\hat{v}_t}}\odot \left(m_t\odot m_t\right)\nonumber\\
\leq & \frac{1-\beta_1}{\sqrt{1-\beta_2}}\sum_{t=p}^{T+p}\sum_{j=1}^d\sum_{i=1}^t\lambda^{t-i}\left\|g_{i,j}\right\|_2\nonumber\\
=&\frac{1-\beta_1}{\sqrt{1-\beta_2}}\sum_{j=1}^d\sum_{t=p}^{T+p}\left(\sum_{i=p+1}^t\lambda^{t-i}\left\|g_{i,j}\right\|_2+\sum_{i=1}^{p}\lambda^{t-i}\left\|g_{i,j}\right\|_2 \right)\nonumber\\
=&\frac{1-\beta_1}{\sqrt{1-\beta_2}}\sum_{j=1}^d\left(\sum_{t=p+1}^{T+p}\sum_{i=p+1}^t\lambda^{t-i}\left\|g_{i,j}\right\|_2+\sum_{t=p}^{T+p}\sum_{i=1}^{p}\lambda^{t-i}\left\|g_{i,j}\right\|_2 \right)\nonumber\\
=&\frac{1-\beta_1}{\sqrt{1-\beta_2}}\sum_{j=1}^d\left(\sum_{t=p+1}^{T+p}\sum_{i=p+1}^t\lambda^{t-i}\left\|g_{i,j}\right\|_2+
\left(\sum_{i=1}^p\lambda^{p-i}\left\|g_{i,j}\right\|_2\right)\left(\sum_{i=0}^T\lambda^i\right) \right).
\end{align}
We first bound the first term in (\ref{ADAMeq:2}) for each $j$ as follows, 
\begin{align}
&\sum_{t=p+1}^{T+p}\sum_{i=p+1}^t\lambda^{t-i}\left\|g_{i,j}\right\|_2=\sum_{i=p+1}^{T+p}\left\|g_{i,j}\right\|_2\sum_{t=i}^{T+p}\lambda^{T+p-t}\nonumber\\
\leq &\frac{1}{1-\lambda}\sum_{t=p+1}^{T+p}\left\|g_{i,j}\right\|_2\leq \frac{TG_{\infty}}{1-\lambda}.
\label{ADAMeq:3}
\end{align}
The first inequality follows from the fact that $\sum_{t=i}^{T+p}\lambda^{T+p-t}<\frac{1}{1-\lambda}$ and the last inequality is due to \ref{as1:2} in Assumption \hyperref[as:1]{1}. Using a similar argument, we further bound the second term in (\ref{ADAMeq:2}) as follows,
\begin{align}
&\left(\sum_{i=1}^p\lambda^{p-i}\left\|g_{i,j}\right\|_2\right)\left(\sum_{i=0}^T\lambda^i\right)\leq \frac{1}{1-\lambda}\left(\sum_{i=1}^p\lambda^{p-i}\left\|g_{i,j}\right\|_2\right)\nonumber\\
\leq &\frac{G_{\infty}}{1-\lambda}\left(\sum_{i=1}^p\lambda^{p-i}\right)\leq  \frac{G_{\infty}}{\left(1-\lambda\right)^2}.
\label{ADAMeq:4}
\end{align}
Inserting (\ref{ADAMeq:3}) and (\ref{ADAMeq:4}) into (\ref{ADAMeq:2}) implies
\begin{align*}
\sum_{t=p}^{T+p}\left\|\frac{1}{\sqrt[4]{\hat{v}_t}}\odot m_t\right\|^2\leq \frac{d\left(1-\beta_1\right)}{\sqrt{1-\beta_2}}\left( \frac{TG_{\infty}}{1-\lambda}+ \frac{G_{\infty}}{\left(1-\lambda\right)^2}\right).
\end{align*}
This completes the proof of the lemma.
\end{proof}
In order to establish the regret analysis of Algorithm \ref{convgADAM}, we further need the following intermediate result.
\begin{lem}
\label{lem2}
Under the conditions in Theorem \hyperref[thm2]{2}, we have
\begin{align*}
\sum_{t=p}^{T+p}\left\|m_{t-1}\right\|^2\leq \mathcal{O}(T).
\end{align*}
\end{lem}
\begin{proof}[Proof of Lemma \ref{lem2}]
By the definition of $m_t$, we obtain
\begin{align*}
&\sum_{t=p}^{T+p}\left\|m_{t-1}\right\|^2=\sum_{t=p}^{T+p}\sum_{j=1}^dm_{t-1,j}^2\\
=&\sum_{t=p}^{T+p}\sum_{j=1}^d\left( (1-\beta_1)\sum_{i=1}^t\beta_1^{t-i}g_{i,j}\right)^2\\
\leq & (1-\beta_1)^2\sum_{t=p}^{T+p} \sum_{j=1}^d\left(\sum_{i=1}^t\beta_1^{t-i} \right)\left(\sum_{i=1}^t\beta_1^{t-i}g_{i,j}^2 \right)\\
\leq &(1-\beta_1)\sum_{t=p}^{T+p} \sum_{j=1}^d\left(\sum_{i=1}^t\beta_1^{t-i}g_{i,j}^2 \right)\leq (1-\beta_1)\sum_{t=p}^{T+p}\sum_{j=1}^d \left(G_{\infty}\sum_{i=1}^t\beta_1^{t-i} \right)\\
\leq &\sum_{t=p}^{T+p}\sum_{j=1}^d G_{\infty}=dTG_{\infty}.
\end{align*}
The first inequality follows from the Cauchy-Schwarz inequality. The second and the last inequalities use the fact that $\sum_{i=1}^t\beta_1^{t-i}\leq \frac{1}{1-\beta_1}$. The third inequality is due to \ref{as1:2} in Assumption \hyperref[as:1]{1}. This completes the proof of the lemma. 
\end{proof}
\subsubsection*{Proof of Theorem \ref{thm2}}
\label{proof:thm2}
\begin{proof}
Based on the update step~\ref{ADAM:update} in Algorithm \ref{convgADAM} and given any $\omega^*\in\mathbb{R}^d$, we obtain
\begin{align}
&\left\|\omega_{t+1}-\omega^*\right\|^2=\left\|\omega_t-\frac{\eta }{\sqrt{\hat{v}_t}}\odot m_t-\omega^*\right\|^2\nonumber\\
=&\left\|\omega_t-\omega^*\right\|^2-2\left\langle \omega_t-\omega^*,\frac{\eta }{\sqrt{\hat{v}_t}}\odot m_t\right\rangle+\left\|\frac{\eta }{\sqrt{\hat{v}_t}}\odot m_t\right\|^2\nonumber\\
=&\left\|\omega_t-\omega^*\right\|^2-2\left\langle \omega_t-\omega^*,\frac{\eta (1-\beta_1) }{\sqrt{\hat{v}_t}}\odot g_t\right\rangle-2\left\langle \omega_t-\omega^*,\frac{\eta \beta_1 }{\sqrt{\hat{v}_t}}\odot m_{t-1} \right\rangle\nonumber\\
&\quad\quad\quad\quad\quad\quad\quad\quad\quad+\left\|\frac{\eta }{\sqrt{\hat{v}_t}\odot m_t}\right\|^2.
\label{ADAMeq:5}
\end{align}
The first inequality uses the same argument as those used in Theorem \ref{thm1}. Rearranging (\ref{ADAMeq:5}) gives
\begin{align}
\label{ADAMeq:6}
\left\langle \omega_t-\omega^*,g_t\right\rangle&= \frac{\left[\left\|\sqrt[4]{\hat{v}_t}\odot(\omega_t-\omega^*)\right\|^2-\left\|\sqrt[4]{\hat{v}_t}\odot(\omega_{t+1}-\omega^*)\right\|^2 \right]}{2\eta(1-\beta_1)}\nonumber\\
&\quad-\frac{\beta_1}{1-\beta_1}\left\langle\frac{\omega_t-\omega^*}{\sqrt{\eta}}, m_{t-1}\sqrt{\eta} \right\rangle + \frac{1}{2\eta(1-\beta_1)}\left\|\frac{\eta }{\sqrt[4]{\hat{v}_t}}\odot m_t \right\|^2\nonumber\\
&\leq \frac{\left[\left\|\sqrt[4]{\hat{v}_t}\odot(\omega_t-\omega^*)\right\|^2-\left\|\sqrt[4]{\hat{v}_t}\odot(\omega_{t+1}-\omega^*)\right\|^2 \right]}{2\eta(1-\beta_1)}\nonumber\\
&\quad+\frac{\beta_1}{1-\beta_1}\left[\frac{\left\|\omega_t-\omega^* \right\|^2}{2\eta}+\frac{m_{t-1}\odot m_{t-1}\eta}{2} \right] +\frac{\eta }{2(1-\beta_1)}\left\|\frac{1}{\sqrt[4]{\hat{v}_t}} \odot m_t\right\|^2.
\end{align}
From the strong convexity property of $f_t$ in \ref{as1:4} in Assumption \hyperref[as:1]{1},  we obtain
\begin{align*}
f_t(\omega_t)-f_t(\omega^*)\leq \left\langle \omega_t-\omega^*,\bigtriangledown f_t(\omega_t)\right\rangle-\frac{H}{2}\left\|\omega_t-\omega^*\right\|^2.
\end{align*}
Using (\ref{ADAMeq:6}) in the above inequality and summing up over all time steps yields
\begin{align}
&\sum_{t=p}^{T+p}\left[f_t(\omega_t)-f_t(\omega^*)\right]\nonumber\\
\leq &\sum_{t=p}^{T+p}\left\{\frac{\left[\left\|\sqrt[4]{\hat{v}_t}\odot(\omega_t-\omega^*)\right\|^2-\left\|\sqrt[4]{\hat{v}_t}\odot(\omega_{t+1}-\omega^*)\right\|^2 \right]}{2\eta(1-\beta_1)}+\left\|\omega_t-\omega^* \right\|^2\left[\frac{\beta_1}{2\eta(1-\beta_1)} -\frac{H}{2}\right]\right.\nonumber\\
&\quad\quad\quad\quad\quad\quad\left.
+\frac{\eta}{2(1-\beta_1)}\left[\beta_1m_{t-1}\odot m_{t-1}+\left\|\frac{1}{\sqrt[4]{\hat{v}_t}}\odot m_t \right\|^2 \right] \right\}.
\label{ADAMeq:7}
\end{align}
We proceed by separating (\ref{ADAMeq:7}) into 3 parts and find upper bounds for each one of them. Considering the first part in (\ref{ADAMeq:7}), we have
\begin{align}
\label{ADAMeq:8}
&\sum_{t=p}^{T+p}\frac{\left[\left\|\sqrt[4]{\hat{v}_t}\odot(\omega_t-\omega^*)\right\|^2-\left\|\sqrt[4]{\hat{v}_t}\odot(\omega_{t+1}-\omega^*)\right\|^2 \right]}{2\eta(1-\beta_1)}\nonumber\\
\leq&\frac{\left\|\sqrt[4]{\hat{v}_p}\odot(\omega_p-\omega^*)\right\|^2}{2\eta(1-\beta_1)}+\frac{1}{2\eta(1-\beta_1)}\sum_{t=p+1}^{T+p}\left(\left\|\sqrt[4]{\hat{v}_t}\odot(\omega_t-\omega^*) \right\|^2\right.\nonumber\\
&\quad\quad\quad\quad\quad\quad\quad\left.-\left\|\sqrt[4]{\hat{v}_{t-1}}\odot(\omega_t-\omega^*) \right\|^2 \right) \nonumber\\
=&\frac{1}{2\eta(1-\beta_1)}\left[\left\|\sqrt[4]{\hat{v}_p}\odot(\omega_p-\omega^*)\right\|^2+\sum_{t=p+1}^{T+p}\left(\sum_{j=1}^d\sqrt{\hat{v}_{t,j}}(\omega_{t,j}-\omega^{*,j})^2\right.\right.\nonumber\\
&\quad\quad\quad\quad\quad\quad\quad\left.\left.-\sum_{j=1}^d\sqrt{\hat{v}_{t-1,j}}(\omega_{t,j}-\omega^{*,j})^2\right) \right]\nonumber\\
=&\frac{1}{2\eta(1-\beta_1)}\left[\left\|\sqrt[4]{\hat{v}_p}\odot(\omega_p-\omega^*)\right\|^2+\sum_{t=p+1}^{T+p}\left(\sum_{j=1}^d(\omega_{t,j}-\omega^{*,j})^2\left( \sqrt{\hat{v}_{t,j}}-\sqrt{\hat{v}_{t-1,j}}\right)\right) \right].
\end{align}
Since $\hat{v}_{t,j}$ is maximum of all $v_{t,j}$ for each $j$ until the current time step, i.e. $\sqrt{\hat{v}_{t,j}}-\sqrt{\hat{v}_{t-1,j}}\geq 0$, by using \ref{as1:1} in Assumption \hyperref[as:1]{1}, (\ref{ADAMeq:8}) can be further bounded as follows,
\begin{align}
&\sum_{t=p}^{T+p}\frac{\left[\left\|\sqrt[4]{\hat{v}_t}\odot(\omega_t-\omega^*)\right\|^2-\left\|\sqrt[4]{\hat{v}_t}\odot(\omega_{t+1}-\omega^*)\right\|^2 \right]}{2\eta(1-\beta_1)}\nonumber\\
\leq& \frac{1}{2\eta(1-\beta_1)}\left[\left\|\sqrt[4]{\hat{v}_p}\odot(\omega_p-\omega^*)\right\|^2+D_{\infty}^2\sum_{j=1}^d\sum_{t=p+1}^{T+p}\left(\sqrt{\hat{v}_{t,j}}-\sqrt{\hat{v}_{t-1,j}}\right)\right]\nonumber\\
\leq& \frac{1}{2\eta(1-\beta_1)}\left[D_{\infty}^2\sum_{j=1}^d\sqrt{\hat{v}_{p,j}}+D_{\infty}^2\sum_{j=1}^d\sum_{t=p+1}^{T+p}\left(\sqrt{\hat{v}_{t,j}}-\sqrt{\hat{v}_{t-1,j}}\right) \right]\nonumber\\
=& \frac{1}{2\eta(1-\beta_1)}D_{\infty}^2\sum_{j=1}^d\sqrt{\hat{v}_{p+T,j}}\nonumber.
\end{align}
By the definition of $\hat{v}_t$ in step~\ref{ADAM:vupdate} in Algorithm \ref{convgADAM}, for any $t$ and $j$, we have
\begin{align*}
v_{t,j}=(1-\beta_2)\sum_{i=1}^t\beta_2^{t-i}g_{i,j}^2\leq (1-\beta_2)G_{\infty}^2\sum_{i=1}^t\beta_2^{t-i}\leq G_{\infty}^2,
\end{align*}
which in turn yields
\begin{align}
\sum_{t=p}^{T+p}&\frac{\left[\left\|\sqrt[4]{\hat{v}_t}\odot(\omega_t-\omega^*)\right\|^2-\left\|\sqrt[4]{\hat{v}_t}\odot(\omega_{t+1}-\omega^*)\right\|^2 \right]}{2\eta(1-\beta_1)}\leq \frac{dD_{\infty}^2G_{\infty}}{2\eta(1-\beta_1)}=\mathcal{O}(\sqrt{T}).
\label{ADAMeq:11}
\end{align}
The last equality is due to the setting of the stepsize, i.e. $\eta=\frac{\eta_1}{\sqrt{T}}$. For the second term in (\ref{ADAMeq:7}), from the relationship between $\beta_1$ and $H$, we obtain
\begin{align*}
\frac{\beta_1}{1-\beta_1}\leq H\eta,
\end{align*}
which in turn yields
\begin{align}
\label{ADAMeq:9}
\frac{\beta_1}{2\eta(1-\beta_1)}-\frac{H}{2}\leq 0.
\end{align}
Thus, (\ref{ADAMeq:9}) guarantees negativity of the second term in (\ref{ADAMeq:7}). For the third term in (\ref{ADAMeq:7}), by using Lemmas \ref{lem1} and \ref{lem2}, we assert
\begin{align}
\label{ADAMeq:10}
\frac{\eta}{2(1-\beta_1)}\left[\beta_1m_{t-1}\odot m_{t-1}+\left\|\frac{1}{\sqrt[4]{\hat{v}_t}}\odot m_t \right\|^2 \right] \leq \mathcal{O}(\frac{1}{\sqrt{T}})\cdot\mathcal{O}(T)=\mathcal{O}(\sqrt{T}).
\end{align}
The desired result follows directly from (\ref{ADAMeq:7}), (\ref{ADAMeq:11}), (\ref{ADAMeq:9}) and (\ref{ADAMeq:10}).
\end{proof}

\subsection*{D\tab Regret with Rolling Window Analysis of dnnOGD for Two-Layer ReLU Neural Network}
For a two-layer ReLU neural network, we first introduce $\mathcal{F}^t$ that records all previous iterates up until $t$.
\begin{lem}
\label{lem:obj}
If conditions \ref{as2:1} and \ref{as2:2} hold from Assumption \hyperref[as:2]{2}, we have
\begin{align}
\mathop{{}\mathbb{E}}\left[l_t\left(\omega_{1,t},\omega_{2,t} \right)\mid\mathcal{F}^t\right] 
=\frac{\rho^2}{2}(\omega_{1,t}^T\omega_{2,t}z^t-\omega_{1,*}^T\omega_{2,*}z^t)^2.
\label{obj:non}
\end{align}
\end{lem}
\begin{proof}[Proof of Lemma \ref{lem:obj}]
Based on condition \ref{as2:2} in Assumption \hyperref[as:2]{2}, we obtain
\begin{align*}
&\mathop{{}\mathbb{E}_{\sigma_1,\sigma_2}}\left[f_t(\omega_{1,t},\omega_{2,t})\mid \mathcal{F}^t \right]\\
=&
\frac{1}{2}
\mathop{{}\mathbb{E}_{\sigma_1}}\left[\omega_{1,t}^T\sigma_1\left(\omega_{2,t} z^t\right)-y^t\mid \mathcal{F}^t \right] \cdot \mathop{{}\mathbb{E}_{\sigma_2}}\left[\omega_{1,t}^T\sigma_2\left(\omega_{2,t} z^t\right)-y^t\mid \mathcal{F}^t \right]\\
=&\frac{1}{2}\left( \rho\omega_{1,t}^T\omega_{2,t}z^t - y^t \right)\cdot \left( \rho\omega_{1,t}^T\omega_{2,t}z^t - y^t \right) =\frac{1}{2} \left( \rho\omega_{1,t}^T\omega_{2,t}z^t - y^t \right)^2\\
=&\frac{\rho^2}{2}(\omega_{1,t}^T\omega_{2,t}z^t-\omega_{1,*}^T\omega_{2,*}z^t)^2.
\end{align*}
On the other hand, we get
\begin{align*}
&\mathop{{}\mathbb{E}_{\sigma_1,\sigma_2}}\left[f_t(\omega_{1,*},\omega_{2,*})\mid \mathcal{F}^t \right]\\
=&
\frac{1}{2}
\mathop{{}\mathbb{E}_{\sigma_1}}\left[\omega_{1,*}^T\sigma_1\left(\omega_{2,*} z^t\right)-y^t\mid \mathcal{F}^t \right] \cdot \mathop{{}\mathbb{E}_{\sigma_2}}\left[\omega_{1,*}^T\sigma_2\left(\omega_{2,*} z^t\right)-y^t\mid \mathcal{F}^t \right]\\
=&\frac{1}{2}
\left(\mathop{{}\mathbb{E}_{\sigma_1}}\left[\omega_{1,*}^T\sigma_1\left(\omega_{2,*} z^t\right)\mid \mathcal{F}^t \right]-y^t\right) \cdot \left(\mathop{{}\mathbb{E}_{\sigma_2}}\left[\omega_{1,*}^T\sigma_2\left(\omega_{2,*} z^t\right)\mid \mathcal{F}^t \right]-y^t\right)\\
=&0,
\end{align*}
which in turn yields
\begin{align*}
&\mathop{{}\mathbb{E}}\left[l_t\left(\omega_{1,t},\omega_{2,t} \right)\mid\mathcal{F}^t\right]\nonumber\\ 
=& \mathop{{}\mathbb{E}_{\sigma_1,\sigma_2}}\left[f_t(\omega_{1,t},\omega_{2,t})\mid \mathcal{F}^t \right]-\mathop{{}\mathbb{E}_{\sigma_1,\sigma_2}}\left[f_t(\omega_{1,*},\omega_{2,*})\mid \mathcal{F}^t \right]\nonumber\\
=&\frac{\rho^2}{2}(\omega_{1,t}^T\omega_{2,t}z^t-\omega_{1,*}^T\omega_{2,*}z^t)^2.
\end{align*}
This completes the proof of the lemma.
\end{proof}

\begin{lem}
\label{lem3}
Under the conditions assumed in Theorem \hyperref[thm3]{3}, we have
\begin{align}
\label{lem3:grad1}
&\mathop{{}\mathbb{E}_{\sigma_1,\sigma_2}}\left[g_{1,t} \mid\mathcal{F}^t\right]=\rho^2\left(\omega_{1,t}^T\omega_{2,t}z^t-\omega_{1,*}^T\omega_{2,*}z^t \right)\omega_{2,t}z^t\\
\label{lem3:grad2}
&\mathop{{}\mathbb{E}_{\sigma_1,\sigma_2}}\left[g_{2,t}\mid\mathcal{F}^t \right]=\rho^2\left(\omega_{1,t}^T\omega_{2,t}z^t-\omega_{1,*}^T\omega_{2,*}z^t \right)\omega_{1,t}\left(z^t\right)^T.
\end{align}
\end{lem}
\begin{proof}[Proof of Lemma \ref{lem3}]
From steps~\ref{dnn:g1} and~\ref{dnn:g2}, we have
\begin{align*}
&\mathop{{}\mathbb{E}_{\sigma_1,\sigma_2}}\left[g_{1,t}\mid\mathcal{F}^t \right]\\
=& 
\mathop{{}\mathbb{E}_{\sigma_1,\sigma_2}}\left[\bigtriangledown_{\omega_1}\left( \frac{1}{2}\left(\omega_{1,t}^T\sigma_1\left( \omega_{2,t}z^t\right)-y^t \right)\left(\omega_{1,t}^T\sigma_2\left( \omega_{2,t}z^t\right)-y^t \right) \right)\mid\mathcal{F}^t \right]\\
=&\mathop{{}\mathbb{E}_{\sigma_1,\sigma_2}}\left[\left(\omega_{1,t}^T\sigma_1\left( \omega_{2,t}z^t\right)-y^t \right)\sigma_2\left(\omega_{2,t}z^t \right)\mid\mathcal{F}^t \right]\\
=&\mathop{{}\mathbb{E}_{\sigma_1}}\left[\omega_{1,t}^T\sigma_1\left( \omega_{2,t}z^t\right)-\omega_{1,*}^T\sigma_1\left( \omega_{2,*}z^t\right) \mid\mathcal{F}^t \right]
\mathop{{}\mathbb{E}_{\sigma_2}}\left[\sigma_2\left(\omega_{2,t}z^t \right)\mid\mathcal{F}^t \right]\\
=&\rho\left( \omega_{1,t}^T \omega_{2,t}z^t-\omega_{1,*}^T\omega_{2,*}z^t\right)\rho\omega_{2,t}z^t  = \rho^2\left(\omega_{1,t}^T\omega_{2,t}z^t-\omega_{1,*}^T\omega_{2,*}z^t \right)\omega_{2,t}z^t.
\end{align*}
Similarly,
\begin{align*}
&\mathop{{}\mathbb{E}_{\sigma_1,\sigma_2}}\left[g_{2t}\mid\mathcal{F}^t \right]\\
=&
\mathop{{}\mathbb{E}_{\sigma_1,\sigma_2}}\left[\bigtriangledown_{\omega_2}\left( \frac{1}{2}\left(\omega_{1,t}^T\sigma_1\left( \omega_{2,t}z^t\right)-y^t \right)\left(\omega_{1,t}^T\sigma_2\left( \omega_{2,t}z^t\right)-y^t \right) \right)\mid\mathcal{F}^t \right]\\
=&\mathop{{}\mathbb{E}_{\sigma_1,\sigma_2}}\left[\left(\omega_{1,t}^T\sigma_1\left( \omega_{2,t}z^t\right)-y^t \right)\omega_{1,t}\left(\sigma_2(z^t) \right)^T\mid\mathcal{F}^t \right] \\
=& \mathop{{}\mathbb{E}_{\sigma_1}}\left[\omega_{1,t}^T\sigma_1\left( \omega_{2,t}z^t\right)-\omega_{1,*}^T\sigma_1\left( \omega_{2,*}z^t\right) \mid\mathcal{F}^t \right]
\mathop{{}\mathbb{E}_{\sigma_2}}\left[\omega_{1,t}\left(\sigma_2(z^t) \right)^T\mid\mathcal{F}^t \right]\\
=&\rho\left( \omega_{1,t}^T\omega_{2,t}z^t-\omega_{1,*}^T\omega_{2,*}z^t\right)\rho\omega_{1,t}\left(z^t\right)^T\\
=&\rho^2\left(\omega_{1,t}^T\omega_{2,t}z^t-\omega_{1,*}^T\omega_{2,*}z^t \right)\omega_{1,t}\left(z^t\right)^T.
\end{align*}
This completes the proof of the lemma.
\end{proof}
\subsection*{Proof of Theorem \ref{thm3}}
\label{proof:thm3}
\begin{proof}
First, 
based on the update step~\ref{nonconvexSGD1} and ~\ref{nonconvexSGD2}  in Algorithm \ref{dnn GD}, we obtain
\begin{align}
&\mathop{{}\mathbb{E}_{\sigma_1,\sigma_2}}\left[\left\|\omega_{1,t+1}^T\omega_{2,t+1}-\omega_{1,*}^T\omega_{2,*} \right\|^2\mid\mathcal{F}^t\right] = \mathop{{}\mathbb{E}_{\sigma_1,\sigma_2}}\left[\left\|\omega_{2,t+1}^T\omega_{1,t+1}-\omega_{2,*}^T\omega_{1,*} \right\|^2\mid\mathcal{F}^t\right]\nonumber\\
=& \mathop{{}\mathbb{E}_{\sigma_1,\sigma_2}} \left[\left\|\left(\omega_{2,t}-\eta g_{2,t}\right)^T\left(\omega_{1,t}-\eta g_{1,t}\right)-\omega_{2,*}^T\omega_{1,*} \right\|^2\mid\mathcal{F}^t\right]\nonumber\\
=&  \mathop{{}\mathbb{E}_{\sigma_1,\sigma_2}} \left[\left\|\omega_{2,t}^T\omega_{1,t}-\omega_{2,*}^T\omega_{1,*} -\eta\left(g_{2,t}^T\omega_{1,t}+\omega_{2,t}^Tg_{1,t} \right)+\eta^2 g_{1,t}^Tg_{2,t}\right\|^2\mid\mathcal{F}^t\right]\nonumber\\
=&\mathop{{}\mathbb{E}_{\sigma_1,\sigma_2}} \left[\left\|\omega_{2,t}^T\omega_{1,t}-\omega_{2,*}^T\omega_{1,*}\right\|^2 -2\eta\left\langle\omega_{2,t}^T\omega_{1,t}-\omega_{2,*}^T\omega_{1,*},  g_{2,t}^T\omega_{1,t}+\omega_{2,t}^Tg_{1,t} \right\rangle\right.\nonumber\\
&\left.\quad\quad+\eta^2\left(2\left\langle\omega_{2,t}^T\omega_{1,t}-\omega_{2,*}^T\omega_{1,*},  g_{2,t}^Tg_{1,t} \right\rangle+\left\| \eta g_{1,t}^Tg_{2,t}-\left(g_{2,t}^T\omega_{1,t}+\omega_{2,t}^Tg_{1,t} \right)\right\|^2\right)\mid\mathcal{F}^t\right]\nonumber\\
=&\left\|\omega_{2,t}^T\omega_{1,t}-\omega_{2,*}^T\omega_{1,*}\right\|^2 - 2\eta\left\langle\omega_{2,t}^T\omega_{1,t}-\omega_{2,*}^T\omega_{1,*},  \mathop{{}\mathbb{E}_{\sigma_1,\sigma_2}}\left[g_{2,t}^T\mid\mathcal{F}^t\right]\omega_{1,t}\right.\nonumber\\
&\quad\quad\left.+\omega_{2,t}^T\mathop{{}\mathbb{E}_{\sigma_1,\sigma_2}}\left[g_{1,t}\mid\mathcal{F}^t\right] \right\rangle
 + \eta^2\mathop{{}\mathbb{E}_{\sigma_1,\sigma_2}}\left[2\left\langle\omega_{2,t}^T\omega_{1,t}-\omega_{2,*}^T\omega_{1,*},  g_{2,t}^Tg_{1,t} \right\rangle\right.\nonumber\\
&\quad\quad\left.+\left\| \eta g_{1,t}^Tg_{2,t}-\left(g_{2,t}^T\omega_{1,t}+\omega_{2,t}^Tg_{1,t} \right)\right\|^2\mid\mathcal{F}^t\right].
\label{thm3:1}
\end{align}
By Lemma \ref{lem3} we conclude that $\mathop{{}\mathbb{E}}\left[\left\|g_{1,t}\right\|\mid\mathcal{F}^t\right]$ and $\mathop{{}\mathbb{E}}\left[\left\|g_{2,t}\right\|\mid\mathcal{F}^t\right]$ are bounded due to \hyperref[as2:3]{3} in Assumption \hyperref[as:2]{2}, which in turn yields
\begin{align}
\mathop{{}\mathbb{E}_{\sigma_1,\sigma_2}}&\left[2\left\langle\omega_{2,t}^T\omega_{1,t}-\omega_{2,*}^T\omega_{1,*},  g_{2,t}^Tg_{1,t} \right\rangle+\left\| \eta g_{1,t}^Tg_{2,t}-\left(g_{2,t}^T\omega_{1,t}+\omega_{2,t}^Tg_{1,t} \right)\right\|^2\mid\mathcal{F}^t\right]\nonumber\\
\leq &\left\|\omega_{2,t}^T\omega_{1,t}-\omega_{2,*}^T\omega_{1,*} \right\|^2\cdot\mathop{{}\mathbb{E}_{\sigma_1,\sigma_2}}\left[ \left\|g_{2,t}^Tg_{1,t}\right\|^2 \mid\mathcal{F}^t\right] \nonumber\\
&\quad\quad\quad\quad\quad+ \mathop{{}\mathbb{E}_{\sigma_1,\sigma_2}}\left[\left\| \eta g_{1,t}^Tg_{2,t}-\left(g_{2,t}^T\omega_{1,t}+\omega_{2,t}^Tg_{1,t} \right)\right\|^2\mid\mathcal{F}^t\right]\nonumber\\
\leq & M_1,
\label{thm3:2}
\end{align}
where $M_1$ is a fixed positive number. The first inequality comes from the Cauchy-Schwarz inequality and the second inequality is due to the boundedness of $\omega_{1,t}$, $\omega_{2,t}$, $\omega_{1,*}$, $\omega_{2,*}$, $g_{1,t}$, $g_{2,t}$ and $\eta$. 
Inserting (\ref{thm3:2}) into (\ref{thm3:1}) gives
\begin{align}
 &\left\langle\omega_{2,t}^T\omega_{1,t}-\omega_{2,*}^T\omega_{1,*},  \mathop{{}\mathbb{E}_{\sigma_1,\sigma_2}}\left[g_{2,t}^T\mid\mathcal{F}^t\right]\omega_{1,t}\right\rangle + \left\langle\omega_{2,t}^T\omega_{1,t}-\omega_{2,*}^T\omega_{1,*},\omega_{2,t}^T\mathop{{}\mathbb{E}_{\sigma_1,\sigma_2}}\left[g_{1,t}\mid\mathcal{F}^t\right] \right\rangle\nonumber\\
 =&\left\langle\omega_{2,t}^T\omega_{1,t}-\omega_{2,*}^T\omega_{1,*},  \mathop{{}\mathbb{E}_{\sigma_1,\sigma_2}}\left[g_{2,t}^T\mid\mathcal{F}^t\right]\omega_{1,t}+\omega_{2,t}^T\mathop{{}\mathbb{E}_{\sigma_1,\sigma_2}}\left[g_{1,t}\mid\mathcal{F}^t\right] \right\rangle\nonumber \\
 \leq&\frac{\left\|\omega_{2,t}^T\omega_{1,t}-\omega_{2,*}^T\omega_{1,*}\right\|^2 - \mathop{{}\mathbb{E}_{\sigma_1,\sigma_2}}\left[\left\|\omega_{1,t+1}^T\omega_{2,t+1}-\omega_{1,*}^T\omega_{2,*} \right\|^2\mid\mathcal{F}^t\right]}{2\eta}+ \frac{\eta M_1}{2}.
 \label{thm3:3}
\end{align}
Using (\ref{lem3:grad1}) yields
\begin{align}
&\left\langle\omega_{2,t}^T\omega_{1,t}-\omega_{2,*}^T\omega_{1,*},  \mathop{{}\mathbb{E}_{\sigma_1,\sigma_2}}\left[g_{2,t}^T\mid\mathcal{F}^t\right]\omega_{1,t}\right\rangle\nonumber\\
=& \left\langle\omega_{2,t}^T\omega_{1,t}-\omega_{2,*}^T\omega_{1,*},  \rho^2\left(\omega_{1,t}^T\omega_{2,t}z^t-\omega_{1,*}^T\omega_{2,*}z^t \right)z^t\omega_{1,t}^T\omega_{1,t}\right\rangle\nonumber\\
=&\rho^2\left(\omega_{1,t}^T\omega_{2,t}z^t-\omega_{1,*}^T\omega_{2,*}z^t \right)\left(\omega_{1,t}^T\omega_{2,t}-\omega_{1,*}^T\omega_{2,*} \right)z^t\left\|\omega_{1,t}\right\|^2\nonumber\\
=&\rho^2\left(\omega_{1,t}^T\omega_{2,t}z^t-\omega_{1,*}^T\omega_{2,*}z^t \right)^2\left\|\omega_{1,t}\right\|^2=\mathop{{}\mathbb{E}}\left[l_t\left(\omega_{1,t},\omega_{2,t} \right)\mid\mathcal{F}^t\right]\cdot 2\left\|\omega_{1,t}\right\|^2.
\label{thm3:4}
\end{align}
The last equality follows from (\ref{obj:non}) in Lemma \ref{lem:obj}. Then, we have
\begin{align}
&\left|\left\langle\omega_{2,t}^T\omega_{1,t}-\omega_{2,*}^T\omega_{1,*},\omega_{2,t}^T\mathop{{}\mathbb{E}_{\sigma_1,\sigma_2}}\left[g_{1,t}\mid\mathcal{F}^t\right] \right\rangle\right|\nonumber\\
=& \left|\left\langle\omega_{2,t}^T\omega_{1,t}-\omega_{2,*}^T\omega_{1,*},\omega_{2,t}^T \rho^2\left(\omega_{1,t}^T\omega_{2,t}z^t-\omega_{1,*}^T\omega_{2,*}z^t \right)\omega_{2,t}z^t \right\rangle\right|\nonumber\\
=&\left| \rho^2\left(\omega_{1,t}^T\omega_{2,t}z^t-\omega_{1,*}^T\omega_{2,*}z^t \right)\left(\omega_{1,t}^T\omega_{2,t}-\omega_{1,*}^T\omega_{2,*}\right)\omega_{2,t}^T\omega_{2,t}z^t\right|\nonumber\\
\leq&\rho^2\left(\omega_{1,t}^T\omega_{2,t}z^t-\omega_{1,*}^T\omega_{2,*}z^t\right)^2\frac{\left\|\omega_{1,t}^T\omega_{2,t}-\omega_{1,*}^T\omega_{2,*}\right\|\left\|\omega_{2,t}^T\omega_{2,t}\right\|\left\|z^t\right\|}{\left|\left(\omega_{1,t}^T\omega_{2,t}-\omega_{1,*}^T\omega_{2,*}\right)z^t\right|} \nonumber\\
\leq&\rho^2\left(\omega_{1,t}^T\omega_{2,t}z^t-\omega_{1,*}^T\omega_{2,*}z^t\right)^2\left\|\omega_{2,t}^T\omega_{2,t}\right\|\frac{\left\|\omega_{1,t}^T\omega_{2,t}-\omega_{1,*}^T\omega_{2,*}\right\|\left\|z^t\right\|}{\left|\left(\omega_{1,t}^T\omega_{2,t}-\omega_{1,*}^T\omega_{2,*}\right)z^t\right|} \nonumber\\
\leq&\rho^2\left(\omega_{1,t}^T\omega_{2,t}z^t-\omega_{1,*}^T\omega_{2,*}z^t\right)^2\frac{\alpha}{\cos\left(\epsilon \right)} =\mathop{{}\mathbb{E}}\left[l_t\left(\omega_{1,t},\omega_{2,t} \right)\mid\mathcal{F}^t\right]\cdot  \frac{2\alpha}{\cos\left(\epsilon \right)}.
\label{thm3:5}
\end{align}
Note that $\left\|\omega_{2,t}^T\omega_{2,t} \right\| = \sigma_{\mathit{max}}(\omega_{2,t}^T)\leq \left\|\omega_{2,t}^T\right\|_F\leq \alpha$ by \ref{as2:3} in Assumption \hyperref[as:2]{2}. If $\left(\omega_{1,t}^T\omega_{2,t}-\omega_{1,*}^T\omega_{2,*}\right)z^t=0$, then the inequality holds trivially. Using (\ref{thm3:3}), (\ref{thm3:4}) and (\ref{thm3:5}) we obtain
\begin{align}
&\mathop{{}\mathbb{E}_{\sigma_1,\sigma_2}}\left[l_t\left(\omega_{1,t},\omega_{2,t} \right)\mid\mathcal{F}^t\right]\cdot 2\left(\left\|\omega_{1,t}\right\|^2-\frac{\alpha}{\cos(\epsilon)} \right)\nonumber\\
\leq& \frac{\left\|\omega_{2,t}^T\omega_{1,t}-\omega_{2,*}^T\omega_{1,*}\right\|^2 - \mathop{{}\mathbb{E}_{\sigma_1,\sigma_2}}\left[\left\|\omega_{1,t+1}^T\omega_{2,t+1}-\omega_{1,*}^T\omega_{2,*} \right\|^2\mid\mathcal{F}^t\right]}{2\eta}+ \frac{\eta M_1}{2}.
\label{thm3:7}
\end{align}
From update step~\ref{nonconvexSGD1} we notice that $\left\|\omega_{1,t} \right\|^2 =\frac{1}{2}+\xi_1= \frac{1}{2}+\frac{\alpha}{\cos(\epsilon)}$, thus, (\ref{thm3:7}) could be further simplified as
\begin{align*}
\mathop{{}\mathbb{E}_{\sigma_1,\sigma_2}}\left[l_t\left(\omega_{1,t},\omega_{2,t} \right)\mid\mathcal{F}^t\right] &\leq \frac{\left\|\omega_{2,t}^T\omega_{1,t}-\omega_{2,*}^T\omega_{1,*}\right\|^2 - \mathop{{}\mathbb{E}_{\sigma_1,\sigma_2}}\left[\left\|\omega_{1,t+1}^T\omega_{2,t+1}-\omega_{1,*}^T\omega_{2,*} \right\|^2\mid\mathcal{F}^t\right]}{2\eta}\\
&\quad\quad\quad\quad+ \frac{\eta M_1}{2}.
\end{align*}
Applying the law of iterated expectation implies
\begin{align*}
\mathop{{}\mathbb{E}}\left[l_t\left(\omega_{1,t},\omega_{2,t}\right)\right] \leq \frac{\mathop{{}\mathbb{E}}\left[\left\|\omega_{2,t}^T\omega_{1,t}-\omega_{2,*}^T\omega_{1,*}\right\|^2\right] - \mathop{{}\mathbb{E}}\left[\left\|\omega_{1,t+1}^T\omega_{2,t+1}-\omega_{1,*}^T\omega_{2,*} \right\|^2\right]}{2\eta}+ \frac{\eta M_1}{2}
\end{align*}
By summing up all differences, we obtain
\begin{align}
\sum_{t=p}^{T+p}&\mathop{{}\mathbb{E}}\left[l_t\left(\omega_{1,t},\omega_{2,t} \right)\right]\leq \frac{1}{2}\sum_{t=p}^{T+p}\frac{\mathop{{}\mathbb{E}}\left[\left\|\omega_{2,t}^T\omega_{1,t}-\omega_{2,*}^T\omega_{1,*}\right\|^2\right] - \mathop{{}\mathbb{E}}\left[\left\|\omega_{1,t+1}^T\omega_{2,t+1}-\omega_{1,*}^T\omega_{2,*} \right\|^2\right]}{\eta}\nonumber\\
&\quad\quad\quad\quad\quad\quad\quad\quad\quad\quad + \frac{ M_1}{2} \eta T\nonumber\\
&=\frac{1}{2}\frac{\mathop{{}\mathbb{E}}\left[\left\|\omega_{2,p}^T\omega_{1,p}-\omega_{2,*}^T\omega_{1,*}\right\|^2\right] - \mathop{{}\mathbb{E}}\left[\left\|\omega_{1,p+T+1}^T\omega_{2,p+T+1}-\omega_{1,*}^T\omega_{2,*} \right\|^2\right]}{\eta}+ \frac{ M_1}{2} \eta T\nonumber\\
& = \mathcal{O}(\sqrt{T}).
\label{thm3:8}
\end{align}
The last equality uses \ref{as2:3} from Assumption \hyperref[as:2]{2} and the definition of $\eta = \frac{\eta_1}{\sqrt{T}}$. The desired result in Theorem \ref{thm3} follows directly from (\ref{thm3:8}) since it holds for any $p$.

\end{proof}
\subsection*{E\tab Regret with Rolling Window Analyses of dnnAdam for Two-Layer ReLU Neural Network}

\begin{lem}
\label{lem4}
In Algorithm \ref{dnn ADAM}, given $\omega_{2,t},\omega_{1,t},\omega_{2,*},\omega_{1,*}$ and $\hat{v}_{2,t}$, there exists a bounded matrix $\tilde{v}_{2,t}$ such that
\begin{align}
\label{lem4:obj}
\left(\sqrt{\hat{v}_{2,t}}\odot \omega_{2,t}\right)^T\omega_{1,t}-\left(\sqrt{\hat{v}_{2,t}}\odot \omega_{2,*}\right)^T\omega_{1,*}=\left(\sqrt{\tilde{v}_{2,t}}\right)^T\left(\omega_{2,t}^T\omega_{1,t}-\omega_{2,*}^T\omega_{1,*}\right),
\end{align}
where $\odot$ is an element-wise multiplication operation.
\end{lem}
\begin{proof}[Proof of Lemma \ref{lem4}]
From step~\ref{dnnAdam:v2} in Algorithm \ref{dnn ADAM}, $v_{2t}$ is a matrix with same value in the same column, which in turn yields
\begin{align*}
\left(\sqrt{\hat{v}_{2,t}}\odot\omega_{2,t}\right)^T\omega_{1t}= \left(\sqrt{\tilde{v}_{2,t}}\right)^T\omega_{2,t}^T\omega_{1,t},
\end{align*}
where $\tilde{v}_{2,t}$ is a diagonal matrix with $\tilde{v}_{2,t}=\mathrm{diag}\left(\left[\hat{v}_{2,t} \right]_{1,:}\right)$, and $\left[\hat{v}_{2,t} \right]_{1,:}$ is the 1$^{st}$ row of matrix $\hat{v}_{2,t}$. Applying the same argument for $\left(\sqrt{\hat{v}_{2,t}}\odot \omega_{2,*}\right)^T\omega_{1,*}$ yields (\ref{lem4:obj}). Next, let us show that $\tilde{v}_{2,t}$ is bounded. It is sufficient to show that $\hat{v}_{2,t}$ is bounded. From steps~\ref{dnnAdam:v2hat} and~\ref{dnnAdam:v2dot}, we conclude that
\begin{align*}
    \hat{v}_{2,t}\leq \max\left(v_{2,1},v_{2,2},\cdots,v_{2,t}\right).
\end{align*}
Therefore, it is sufficient to show that $\dot{v}_{2,t}$ is bounded for all $t$. For each entry in the matrix, since $\left|\left[g_{2,t}\right]_{ij}\right|\leq G_{2,\infty}$, we obtain
\begin{align}
\label{lem4:1}
\left|\left[ \dot{v}_{2,t}\right]_{ik}\right| =\left| (1-\beta_{22})\sum_{j=1}^{t}\beta_{22}^{t-j}\left(\max_{p}\left[g_{2,j}\right]_{pk}^2\right) \right|\leq \left| (1-\beta_{22})\sum_{j=1}^{t}\beta_{22}^{t-j}G_{2,\infty}^2 \right|\leq G_{2,\infty}^2.
\end{align}
By combining with the fact that $g_{2}$ is bounded due to step~\ref{ddnAdam:g2} and the boundedness of $\omega_{1,t}, \omega_{2,t}, z^t$ and $y^t$ from condition \ref{as2:3} in Assumption \hyperref[as:2]{2}, Lemma \ref{lem4} follows. 
\end{proof}

\begin{lem}
\label{lem5}
In Algorithm \ref{dnn ADAM}, given $m_{1,t-1}, m_{1,t}, \hat{v}_{1,t}\in\mathbb{R}^n$ and $m_{2,t}, \hat{v}_{2,t}\in\mathbb{R}^{n\times d}$ for any $t$, and $\beta_{111},\beta_{121},\beta_{21}$ and $\beta_{22}$ are constants between $0$ and $1$ such that $\lambda_1:=\frac{\beta_{111}}{\beta_{21}}< 1$ and $\lambda_2:=\frac{\beta_{121}}{\beta_{22}}< 1$, then
\begin{align}
    &\left\|\frac{1}{\sqrt{\hat{v}_{1,t}}}\odot m_{1,t-1}\right\|^2 \leq \frac{n}{\left(1-\beta_{111}\right)\left(1-\beta_{21} \right)\left(1-\lambda_1\right)}\label{lem5:bd1}\\
    &\left\|\frac{1}{\sqrt{\hat{v}_{1,t}}}\odot m_{1,t}\right\|^2 \leq \frac{n}{\left(1-\beta_{111}\right)\left(1-\beta_{21} \right)\left(1-\lambda_1\right)} \label{lem5:bd2}\\
    &\left\|\frac{1}{\sqrt{\hat{v}_{2,t}}}\odot m_{2,t}\right\|^2 \leq \frac{nd}{\left(1-\beta_{121}\right)\left(1-\beta_{21}\right)\left(1-\lambda_2\right)}\label{lem5:bd3}\\
    &\left\|\frac{1}{\sqrt[4]{\hat{v}_{2,t}}}\odot m_{2,t}\right\|^2 \leq \frac{ndG_{2,\infty}}{\left(1-\beta_{121}\right)\sqrt{1-\beta_{21}}\left(1-\lambda_2\right)}\label{lem5:bd4}.
\end{align}
\end{lem}

\begin{proof}[Proof of Lemma \ref{lem5}]
Based on steps~\ref{dnnAdam:m1} -~\ref{dnnAdam:v2hat} in Algorithm \ref{dnn ADAM}, we obtain
\begin{align}
    &m_{1,t} = \sum_{j=1}^t\left[\left(1-\beta_{11j}\right)\prod_{k=1}^{t-j}\beta_{11\left(t-k+1 \right)}g_{1,j}\right]\label{lem5:m1}\\
    &m_{2,t} = \sum_{j=1}^t\left[\left(1-\beta_{12j}\right)\prod_{k=1}^{t-j}\beta_{12\left(t-k+1 \right)}g_{2,j}\right]\label{lem5:m2}\\
    &\hat{v}_{1,t} \geq \left(1-\beta_{21}\right)\sum_{j=1}^t\beta_{21}^{t-j}g_{1,j}\odot g_{1,j}\label{lem5:v1}\\
    &\hat{v}_{2,t} \geq \left(1-\beta_{22}\right)\sum_{j=1}^t\beta_{22}^{t-j}g_{2,j}\odot g_{2,j}\label{lem5:v2}.
\end{align}
Then, combining (\ref{lem5:m1}) and (\ref{lem5:v1}) yields
\begin{align*}
    \left\|\frac{1}{\sqrt{\hat{v}_{1,t}}}\odot m_{1,t}\right\|^2 &\leq \sum_{i=1}^n\frac{\left(\sum_{j=1}^t\left[\left(1-\beta_{11j}\right)\prod_{k=1}^{t-j}\beta_{11\left(t-k+1 \right)}\left[g_{1,j}\right]_i\right] \right)^2}{ \left(\left(1-\beta_{21}\right)\sum_{j=1}^t\beta_{21}^{t-j}\left[g_{1,j}\right]_i^2 \right)}\\
    &\leq \sum_{i=1}^n\frac{\left(\sum_{j=1}^t\prod_{k=1}^{t-j}\beta_{11\left(t-k+1 \right)}\left[g_{1,j}\right]_i \right)^2}{ \left(\left(1-\beta_{21}\right)\sum_{j=1}^t\beta_{21}^{t-j}\left[g_{1,j}\right]_i^2 \right)}\\
    &\leq \sum_{i=1}^n\frac{\left(\sum_{j=1}^t\prod_{k=1}^{t-j}\beta_{11\left(t-k+1 \right)} \right)
    \left(\sum_{j=1}^t\prod_{k=1}^{t-j}\beta_{11\left(t-k+1 \right)}\left[g_{1,j}\right]_i^2 \right)
    }{ \left(\left(1-\beta_{21}\right)\sum_{j=1}^t\beta_{21}^{t-j}\left[g_{1,j}\right]_i^2 \right)}\\
    &\leq \sum_{i=1}^n\frac{\left(\sum_{j=1}^t\beta_{111}^{t-j} \right)
    \left(\sum_{j=1}^t\beta_{111}^{t-j}\left[g_{1,j}\right]_i^2 \right)
    }{ \left(\left(1-\beta_{21}\right)\sum_{j=1}^t\beta_{21}^{t-j}\left[g_{1,j}\right]_i^2 \right)}\\
    &\leq \frac{1}{\left(1-\beta_{111}\right)\left(1-\beta_{21} \right)}
    \sum_{i=1}^n\sum_{j=1}^t\frac{
    \beta_{111}^{t-j}\left[g_{1,j}\right]_i^2 }{ \beta_{21}^{t-j}\left[g_{1,j}\right]_i^2 }\\
    &\leq \frac{1}{\left(1-\beta_{111}\right)\left(1-\beta_{21} \right)}
    \sum_{i=1}^n\sum_{j=1}^t\lambda_1^{t-j}\\
    &\leq \frac{n}{\left(1-\beta_{111}\right)\left(1-\beta_{21} \right)\left(1-\lambda_1\right)}.
\end{align*}
The first inequality follows from the definition of $\hat{v}_{1,t}$, which is maximum of all $v_{1,t}$ until the current time step. The third inequality follows from the Cauchy-Schwarz inequality and the forth inequality uses the fact that $\beta_{11t}\leq\beta_{111}$ for any $t$. Applying the same argument to $\left\|\frac{1}{\sqrt{\hat{v}_{2,t}}}\odot m_{2,t}\right\|^2$ implies (\ref{lem5:bd3}). Then, applying the fact that $\hat{v}_{1,t}\geq \hat{v_{1,t-1}}$ yields
\begin{align*}
    \left\|\frac{1}{\sqrt{\hat{v}_{1,t}}}\odot m_{1,t-1}\right\|^2 &\leq \left\|\frac{1}{\sqrt{\hat{v}_{1,t-1}}}\odot m_{1,t-1}\right\|^2\leq \frac{n}{\left(1-\beta_{111}\right)\left(1-\beta_{21} \right)\left(1-\lambda_1\right)},
\end{align*}
where the last inequality follows from (\ref{lem5:bd2}). Lastly, $\lambda_2 = \frac{\beta_{121}}{\beta_{22}}<1$ implies $\frac{\beta_{121}}{\sqrt{\beta_{22}}}<\lambda_2 < 1$. By combining (\ref{lem5:m2}) and (\ref{lem5:v2}), we get
\begin{align*}
    \left\|\frac{1}{\sqrt[4]{\hat{v}_{2,t}}}\odot m_{2,t}\right\|^2&\leq \sum_{p=1}^n\sum_{q=1}^d\frac{\left(\sum_{j=1}^t\left[\left(1-\beta_{12j}\right)\prod_{k=1}^{t-j}\beta_{12\left(t-k+1 \right)}\left[g_{2,j}\right]_{pq}\right] \right)^2}{\sqrt{\left(\left(1-\beta_{21}\right)\sum_{j=1}^t\beta_{22}^{t-j}\left[g_{2,j}\right]_{pq}^2 \right)}}\\
    &\leq \sum_{p=1}^n\sum_{q=1}^d\frac{\left(\sum_{j=1}^t\left[\prod_{k=1}^{t-j}\beta_{12\left(t-k+1 \right)}\left[g_{2,j}\right]_{pq}\right] \right)^2}{\sqrt{\left(\left(1-\beta_{21}\right)\sum_{j=1}^t\beta_{22}^{t-j}\left[g_{2,j}\right]_{pq}^2 \right)}}\\
    &\leq \sum_{p=1}^n\sum_{q=1}^d\frac{\left(\sum_{j=1}^t\prod_{k=1}^{t-j}\beta_{12\left(t-k+1 \right)} \right)
    \left(\sum_{j=1}^t\prod_{k=1}^{t-j}\beta_{12\left(t-k+1 \right)}\left[g_{2,j}\right]_{pq}^2 \right)
    }{\sqrt{\left(\left(1-\beta_{21}\right)\sum_{j=1}^t\beta_{22}^{t-j}\left[g_{2,j}\right]_{pq}^2 \right)}}\\
    &\leq \sum_{p=1}^n\sum_{q=1}^d\frac{\left(\sum_{j=1}^t\beta_{121}^{t-j} \right)
    \left(\sum_{j=1}^t\beta_{121}^{t-j}\left[g_{2,j}\right]_{pq}^2 \right)
    }{\sqrt{\left(\left(1-\beta_{21}\right)\sum_{j=1}^t\beta_{22}^{t-j}\left[g_{2,j}\right]_{pq}^2 \right)}}\\
    &\leq\frac{1}{\left(1-\beta_{121}\right)\sqrt{1-\beta_{21}}}\sum_{p=1}^n\sum_{q=1}^d\sum_{j=1}^t\frac{
    \beta_{121}^{t-j}\left[g_{2,j}\right]_{pq}^2 
    }{\sqrt{\beta_{22}^{t-j}\left[g_{2,j}\right]_{pq}^2 }}\\
    &\leq\frac{1}{\left(1-\beta_{121}\right)\sqrt{1-\beta_{21}}}\sum_{p=1}^n\sum_{q=1}^d\sum_{j=1}^t\lambda_2^{t-j}\left|\left[g_{2,j}\right]_{pq}\right|\\
    &\leq\frac{ndG_{2,\infty}}{\left(1-\beta_{121}\right)\sqrt{1-\beta_{21}}\left(1-\lambda_2\right)}.
\end{align*}
\end{proof}

\subsection*{Proof of Theorem \ref{thm4}}
\label{proof:thm4}
\begin{proof}
Now, let us multiply $\left\|\omega_{2,t+1}^T\omega_{1,t+1}-\omega_{2*}^T\omega_{1*}\right\|^2$ by $\sqrt{\hat{v}_{2,t}}$, then take expectation given all records until time $t$. Then, from steps~\ref{dnnAdam:m1} -~\ref{dnnAdam:w2}, we obtain
\begin{align}
    &\mathop{{}\mathbb{E}}\left[\left\|\left(\sqrt[4]{\hat{v}_{2,t}}\odot\omega_{2,t+1}\right)^T\omega_{1,t+1}-\left(\sqrt[4]{\hat{v}_{2,t}}\odot\omega_{2,*}\right)^T\omega_{1,*} \right\|^2\mid\mathcal{F}^t \right]\label{thm4:1}\\
    = &\mathop{{}\mathbb{E}}\left[\left\|\left(\sqrt[4]{\hat{v}_{2,t}}\odot\left(\omega_{2,t}-\frac{\eta }{\sqrt{\hat{v}_{2,t}}}\odot m_{2,t} \right)\right)^T\left(\omega_{1,t}-\frac{\eta }{\sqrt{\hat{v}_{1,t}}}\odot m_{1,t} \right)\right.\right.\nonumber\\
    &\left.\left. -\left(\sqrt[4]{\hat{v}_{2,t}}\odot\omega_{2,*}\right)^T\omega_{1,*} \right\|^2\mid\mathcal{F}^t \right]\nonumber\\
    =&\mathop{{}\mathbb{E}}\left[\left\|\left(\sqrt[4]{\hat{v}_{2,t}}\odot\omega_{2,t}\right)^T\omega_{1,t}-\left(\sqrt[4]{\hat{v}_{2,t}}\odot\omega_{2,*}\right)^T\omega_{1,*} \right\|^2\mid\mathcal{F}^t \right]\nonumber\\
    & -2\mathop{{}\mathbb{E}}\left[\left\langle \left(\sqrt{\hat{v}_{2,t}}\odot\omega_{2,t}\right)^T\omega_{1,t}-\left(\sqrt{\hat{v}_{2,t}}\odot\omega_{2,*}\right)^T\omega_{1*},\omega_{2,t}^T\frac{\eta}{\sqrt{\hat{v}_{1,t}}}\odot m_{1,t}  \right\rangle \mid\mathcal{F}^t\right]\label{thm4:item1}\\
    & -2\mathop{{}\mathbb{E}}\left[\left\langle \left(\sqrt{\hat{v}_{2,t}}\odot\omega_{2,t}\right)^T\omega_{1,t}-\left(\sqrt{\hat{v}_{2,t}}\odot\omega_{2,*}\right)^T\omega_{1,*},\left(\frac{\eta}{\sqrt{\hat{v}_{2,t}}}\odot m_{2,t} \right)^T\omega_{1,t} \right\rangle \mid\mathcal{F}^t\right]\label{thm4:item2}\\
    &  + 2\eta^2\mathop{{}\mathbb{E}}\left[\left\langle \left(\sqrt{\hat{v}_{2,t}}\odot\omega_{2,t}\right)^T\omega_{1,t}-\left(\sqrt{\hat{v}_{2,t}}\odot\omega_{2,*}\right)^T\omega_{1,*},\right.\right.\nonumber\\
    &\quad\quad\quad\left.\left.\left(\frac{\eta }{\sqrt{\hat{v}_{2,t}}}\odot m_{2,t} \right)^T\left(\frac{\eta}{\sqrt{\hat{v}_{1,t}}}\odot m_{1,t} \right) \right\rangle \mid\mathcal{F}^t\right]\label{thm4:item3}\\
    &  + \eta^2\mathop{{}\mathbb{E}}\left[\left\|
    \left(\sqrt[4]{\hat{v}_{2,t}}\odot\omega_{2,t}\right)^T\left(\frac{\eta}{\sqrt{\hat{v}_{1,t}}}\odot m_{1,t} \right)+
    \left(\sqrt[4]{\hat{v}_{2,t}}\odot\left(\frac{\eta}{\sqrt{\hat{v}_{2,t}}}\odot m_{2,t} \right)\right)^T\omega_{1,t}+\right.\right.\nonumber\\
    &\left.\left.\quad\quad\quad\quad\left(\sqrt[4]{\hat{v}_{2,t}}\odot\left(\frac{\eta}{\sqrt{\hat{v}_{2,t}}}\odot m_{2,t} \right)\right)^T\left(\frac{\eta}{\sqrt{\hat{v}_{1,t}}}\odot m_{1,t} \right)\right\|^2 \mid\mathcal{F}^t\right]\label{thm4:item4}.
\end{align}
Let us first consider the expectations in (\ref{thm4:item3}) and (\ref{thm4:item4}). From (\ref{lem4:1}), we conclude that $\hat{v}_{2,t}$ is bounded. Similarly, given $\beta_{11t}=\beta_{111}\gamma_{1}^t$ and $\beta_{12t}=\beta_{121}\gamma_{2}^t$ with $0<\gamma_1,\gamma_2 < 1$, for each entry, we attain
\begin{align*}
    &\left|\left[m_{1,t}\right]_i\right| \leq \left|(1-\beta_{111})\sum_{j=1}^t\beta_{111}^{t-j}\left[g_{1,j}\right]_i \right|\leq \max_j\left|\left[g_{1,j}\right]_i\right|\\
    &\left|\left[m_{2,t}\right]_{ik}\right| \leq \left|(1-\beta_{121})\sum_{j=1}^t\beta_{121}^{t-j}\left[g_{2,j}\right]_{ik} \right|\leq \max_j\left|\left[g_{2,j}\right]_{ik}\right|.
\end{align*}
Since $\left\|\frac{ 1}{\sqrt{\hat{v}_{1,t}}}\odot m_{1,t} \right\|^2$, $\left\|\frac{1}{\sqrt{\hat{v}_{2,t}}}\odot  m_{2,t} \right\|^2$ and $\left\|\frac{ 1}{\sqrt[4]{\hat{v}_{2,t}}} \odot m_{2,t}\right\|^2$ are bounded from Lemma \ref{lem5} and $\omega_{1,t}, \omega_{2,t},\omega_{1,*}, \omega_{2,*}, \hat{v}_{2,t}$ are also bounded from Assumption \hyperref[as:2]{2} and Lemma \ref{lem4}, 
applying Lemma \ref{lem5} and Cauchy-Schwarz inequality yields
\begin{align}
 &2\eta^2\mathop{{}\mathbb{E}}\left[\left\langle \left(\sqrt{\hat{v}_{2,t}}\odot\omega_{2,t}\right)^T\omega_{1,t}-\left(\sqrt{\hat{v}_{2,t}}\odot\omega_{2,*}\right)^T\omega_{1,*},\left(\frac{\eta }{\sqrt{\hat{v}_{2,t}}}\odot m_{2,t} \right)^T\right.\right.\nonumber\\
 &\quad\left.\left.\left(\frac{\eta}{\sqrt{\hat{v}_{1,t}}} \odot m_{1,t}\right) \right\rangle \mid\mathcal{F}^t\right]
  + \eta^2\mathop{{}\mathbb{E}}\left[\left\|
    \left(\sqrt[4]{\hat{v}_{2,t}}\odot\omega_{2,t}\right)^T\left(\frac{\eta}{\sqrt{\hat{v}_{1,t}}}\odot m_{1,t} \right)\right.\right.\nonumber\\
    &\quad\left.\left.+\left(\sqrt[4]{\hat{v}_{2,t}}\odot\left(\frac{\eta}{\sqrt{\hat{v}_{2,t}}}\odot m_{2,t} \right)\right)^T\omega_{1,t}+\left(\sqrt[4]{\hat{v}_{2,t}}\odot\left(\frac{\eta}{\sqrt{\hat{v}_{2,t}}}\odot m_{2,t} \right)\right)^T\right.\right.\nonumber\\
    &\quad\left.\left.\quad\quad\quad\quad\quad\quad\left(\frac{\eta}{\sqrt{\hat{v}_{1,t}}}\odot m_{1,t} \right)\right\|^2 \mid\mathcal{F}^t\right]\nonumber\\
    =&2\eta^2\mathop{{}\mathbb{E}}\left[\left\langle \left(\sqrt{\hat{v}_{2,t}}\odot\omega_{2,t}\right)^T\omega_{1,t}-\left(\sqrt{\hat{v}_{2,t}}\odot\omega_{2,*}\right)^T\omega_{1,*},\left(\frac{\eta }{\sqrt{\hat{v}_{2,t}}}\odot m_{2,t} \right)^T\right.\right.\nonumber\\
    &\quad\left.\left.\left(\frac{\eta }{\sqrt{\hat{v}_{1,t}}}\odot m_{1,t} \right) \right\rangle \mid\mathcal{F}^t\right] + \eta^2\mathop{{}\mathbb{E}}\left[\left\|
    \left(\sqrt[4]{\hat{v}_{2,t}}\odot\omega_{2,t}\right)^T\left(\frac{\eta}{\sqrt{\hat{v}_{1,t}}}\odot m_{1,t} \right)\right.\right.\nonumber\\
    &\quad\left.\left.+
    \left(\frac{\eta}{\sqrt[4]{\hat{v}_{2,t}}}\odot m_{2,t} \right)^T\omega_{1,t}+ \left(\frac{\eta}{\sqrt[4]{\hat{v}_{2,t}}}\odot m_{2,t} \right)^T\left(\frac{\eta}{\sqrt{\hat{v}_{1,t}}}\odot m_{1,t} \right)\right\|^2 \mid\mathcal{F}^t\right]\nonumber\\
    \leq &\eta^2\mathop{{}\mathbb{E}}\left[\left\|\left(\sqrt{\hat{v}_{2,t}}\odot\omega_{2,t}\right)^T\omega_{1,t}-\left(\sqrt{\hat{v}_{2,t}}\odot\omega_{2,*}\right)^T\omega_{1,*},\right\|^2+\left\|\frac{\eta }{\sqrt{\hat{v}_{2,t}}}\odot m_{2,t} \right\|^2\right.\nonumber\\
    &\quad\left.\cdot\left\|\frac{\eta }{\sqrt{\hat{v}_{1,t}}}\odot m_{1,t} \right\|^2 \mid\mathcal{F}^t\right]+ 2\eta^2\mathop{{}\mathbb{E}}\left[\left\|
    \sqrt[4]{\hat{v}_{2,t}}\odot\omega_{2,t}\right\|^2 \left\|\frac{\eta}{\sqrt{\hat{v}_{1,t}}}\odot m_{1,t} \right\|^2\right.\nonumber\\
    &\quad\left.+
    \left\|\frac{\eta}{\sqrt[4]{\hat{v}_{2,t}}}\odot m_{2,t} \right\|^2\left\|\omega_{1,t}\right\|^2+\left\|\frac{\eta}{\sqrt[4]{\hat{v}_{2,t}}}\odot m_{2,t} \right\|^2\left\|\frac{\eta}{\sqrt{\hat{v}_{1,t}}}\odot m_{1,t} \right\|^2 \mid\mathcal{F}^t\right]\nonumber\\
    &\leq \eta^2\cdot M_1,\label{thm4:item3+4}
\end{align}
where $M_1$ is a fixed constant. Now, let us proceed to show an upper bound for the term in (\ref{thm4:item1}). Applying Lemma \ref{lem4} to (\ref{thm4:item1}) yields
\begin{align}
   &\mathop{{}\mathbb{E}}\left[\left\langle \left(\sqrt[2]{\hat{v}_{2,t}}\odot\omega_{2,t}\right)^T\omega_{1,t}-\left(\sqrt[2]{\hat{v}_{2,t}}\odot\omega_{2,*}\right)^T\omega_{1,*},\omega_{2,t}^T\frac{\eta}{\sqrt{\hat{v}_{1,t}}}\odot m_{1,t}  \right\rangle \mid\mathcal{F}^t\right]\nonumber \\
   =&\mathop{{}\mathbb{E}}\left[\left\langle \left(\sqrt{\tilde{v}_{2,t}}\right)^T\left(\omega_{2,t}^T\omega_{1,t}-\omega_{2,*}^T\omega_{1,*}\right),\omega_{2,t}^T\frac{\eta}{\sqrt{\hat{v}_{1,t}}}\odot m_{1,t}  \right\rangle \mid\mathcal{F}^t\right]\nonumber\\
   =&\mathop{{}\mathbb{E}}\left[\left\langle \left(\sqrt{\tilde{v}_{2,t}}\right)^T\left(\omega_{2,t}^T\omega_{1,t}-\omega_{2,*}^T\omega_{1,*}\right),\omega_{2,t}^T\frac{\eta}{\sqrt{\hat{v}_{1,t}}}\odot \left(\beta_{11t}m_{1,t-1}+(1-\beta_{11t})g_{1,t} \right)  \right\rangle \mid\mathcal{F}^t\right]\nonumber\\
   =&\mathop{{}\mathbb{E}}\left[\left\langle \left(\sqrt{\tilde{v}_{2,t}}\right)^T\left(\omega_{2,t}^T\omega_{1,t}-\omega_{2,*}^T\omega_{1,*}\right),\omega_{2,t}^T\frac{\eta}{\sqrt{\hat{v}_{1,t}}}\odot \beta_{11t}m_{1,t-1} \right\rangle \mid\mathcal{F}^t\right]\label{thm4:item1.1}\\
   &\quad\quad\quad+
   \mathop{{}\mathbb{E}}\left[\left\langle \left(\sqrt{\tilde{v}_{2,t}}\right)^T\left(\omega_{2,t}^T\omega_{1,t}-\omega_{2,*}^T\omega_{1,*}\right),\omega_{2,t}^T\frac{\eta}{\sqrt{\hat{v}_{1,t}}}\odot (1-\beta_{11t})g_{1,t}  \right\rangle \mid\mathcal{F}^t\right]\label{thm4:item1.2}.
\end{align}
Since $\omega_{2,t}, \omega_{1,t},  \omega_{2,*}, \omega_{1,*}, \frac{m_{1,t-1}}{\sqrt{\hat{v}_{1,t}}},\tilde{v}_{2,t}$ and $m_{1,t-1}$ are all bounded, for the term in (\ref{thm4:item1.1}), there exists a constant $M_2$ such that
\begin{align}
   &\mathop{{}\mathbb{E}}\left[\left\langle \left(\sqrt{\tilde{v}_{2,t}}\right)^T\left(\omega_{2,t}^T\omega_{1,t}-\omega_{2,*}^T\omega_{1,*}\right),\omega_{2,t}^T\frac{\eta}{\sqrt{\hat{v}_{1,t}}}\odot \beta_{11t}m_{1,t-1} \right\rangle \mid\mathcal{F}^t\right]\nonumber\\
   =&\eta\beta_{11t}\mathop{{}\mathbb{E}}\left[ \left(\omega_{1,t}^T\omega_{2,t}-\omega_{1,*}^T\omega_{2,*}\right)\sqrt{\tilde{v}_{2,t}}\omega_{2,t}^T\frac{1}{\sqrt{\hat{v}_{1,t}}}\odot m_{1,t-1} \mid\mathcal{F}^t\right]\nonumber\\
   \leq&\frac{\eta\beta_{11t}}{2}\mathop{{}\mathbb{E}}\left[\left\| \left(\omega_{1,t}^T\omega_{2,t}-\omega_{1,*}^T\omega_{2,*}\right)\sqrt{\tilde{v}_{2,t}}\omega_{2,t}^T\right\|^2+\left\|\frac{1}{\sqrt{\hat{v}_{1,t}}}m_{1,t-1}\right\|^2 \mid\mathcal{F}^t\right]
   \leq \eta\beta_{11t}M_2.\label{thm4:item1.1result}
\end{align}
Next, let us bound the term in (\ref{thm4:item1.2}). Based on Lemma \ref{lem5}, we have
\begin{align*}
    &\left|\mathop{{}\mathbb{E}}\left[\left\langle \left(\sqrt{\tilde{v}_{2,t}}\right)^T\left(\omega_{2,t}^T\omega_{1,t}-\omega_{2,*}^T\omega_{1,*}\right),\omega_{2,t}^T\frac{\eta}{\sqrt{\hat{v}_{1,t}}}\odot (1-\beta_{11t})g_{1,t}  \right\rangle \mid\mathcal{F}^t\right]\right|\\
    =&\eta(1-\beta_{11t})\left|\mathop{{}\mathbb{E}}\left[\left(\omega_{1,t}^T\omega_{2,t}-\omega_{1,*}^T\omega_{2,*}\right)\sqrt{\tilde{v}_{2,t}}\omega_{2,t}^T\frac{1 }{\sqrt{\hat{v}_{1,t}}}\odot g_{1,t} \mid\mathcal{F}^t\right]\right|\\
    \leq &\eta(1-\beta_{11t})\left\| \omega_{1,t}^T\omega_{2,t}-\omega_{1,*}^T\omega_{2,*}\right\|\mathop{{}\mathbb{E}}\left[\left\|\sqrt{\tilde{v}_{2,t}}\omega_{2,t}^T\frac{1 }{\sqrt{\hat{v}_{1,t}}} \odot g_{1,t}\right\|   \mid\mathcal{F}^t\right].
\end{align*}
Now, let us focus on the product in the expectation. Since $\sqrt{\tilde{v}_{2,t}}\in\mathbb{R}^{d\times d}$ is a diagonal matrix, let us denote the $i_{\mathrm{th}}$ element on diagonal as $\left[\tilde{v}_{2,t}\right]_i$. Then,
\begin{align*}
   \sqrt{\tilde{v}_{2,t}}\omega_{2,t}^T\frac{1 }{\sqrt{\hat{v}_{1,t}}}\odot g_{1,t} = \left(\mathcal{V}_{t}\odot\omega_{2,t}\right)^Tg_{1,t},
\end{align*}
where $\mathcal{V}_{12}\in\mathbb{R}^{n\times d}$ such that $\left[\mathcal{V}_{12}\right]_{ij} = \sqrt{\frac{\left[\tilde{v}_{2,t}\right]_j}{\left[\hat{v}_{1,t} \right]_i}}$. Then, we obtain
\begin{align*}
    &\left|\mathop{{}\mathbb{E}}\left[\left\langle \left(\sqrt{\tilde{v}_{2,t}}\right)^T\left(\omega_{2,t}^T\omega_{1,t}-\omega_{2,*}^T\omega_{1,*}\right),\omega_{2,t}^T\frac{\eta}{\sqrt{\hat{v}_{1,t}}}\odot (1-\beta_{11t})g_{1,t}  \right\rangle \mid\mathcal{F}^t\right]\right|\\
    \leq &\eta(1-\beta_{11t})\left\| \omega_{1,t}^T\omega_{2,t}-\omega_{1,*}^T\omega_{2,*}\right\|\mathop{{}\mathbb{E}}\left[\left\|\left(\mathcal{V}_{t}\odot\omega_{2,t}\right)\right\|\left\|g_{1,t} \right\|   \mid\mathcal{F}^t\right].
\end{align*}
Based on (\ref{lem4:1}) and condition \ref{as2:5} from Assumption \hyperref[as:2]{2}, we discover
\begin{align}
    \left[\mathcal{V}_{12}\right]_{ij} = \sqrt{\frac{\left[\tilde{v}_{2,t}\right]_j}{\left[\hat{v}_{1,t} \right]_i}}\leq \frac{G_{2,\infty}}{\mu},
    \label{thm4:7}
\end{align}
which in turn yields
\begin{align}
    &\left|\mathop{{}\mathbb{E}}\left[\left\langle \left(\sqrt{\tilde{v}_{2,t}}\right)^T\left(\omega_{2,t}^T\omega_{1,t}-\omega_{2,*}^T\omega_{1,*}\right),\omega_{2,t}^T\frac{\eta}{\sqrt{\hat{v}_{1,t}}}\odot (1-\beta_{11t})g_{1,t}  \right\rangle \mid\mathcal{F}^t\right]\right|\nonumber\\
    \leq &\eta(1-\beta_{11t})\frac{\alpha G_{2,\infty}}{\mu}\left\| \omega_{1,t}^T\omega_{2,t}-\omega_{1,*}^T\omega_{2,*}\right\|\mathop{{}\mathbb{E}}\left[\left\|g_{1,t} \right\|   \mid\mathcal{F}^t\right].\label{thm4:8}
\end{align}
Note that in (\ref{thm4:7}), we assume that $\left[\hat{v}_{1,t} \right]_i$ is nonzero on the $i_{\mathrm{th}}$ coordinate. On the other hand, if $\left[\hat{v}_{1,t} \right]_i$ is zero on the $i_{\mathrm{th}}$ coordinate, then it implies $\left[g_{1,j}\right]_{i}=0$ for $j=1,2,\cdots,t$ on the $i_{\mathrm{th}}$ coordinate, which in turn yields $\left[g_{1,t}\right]_{i} = 0$. Thus, (\ref{thm4:8}) directly follows. Then, based on step~\ref{ddnAdam:g1} in Algorithm \ref{dnn ADAM}, we obtain 
\begin{align}
    &\left|\mathop{{}\mathbb{E}}\left[\left\langle \left(\sqrt{\tilde{v}_{2,t}}\right)^T\left(\omega_{2,t}^T\omega_{1,t}-\omega_{2,*}^T\omega_{1,*}\right),\omega_{2,t}^T\frac{\eta}{\sqrt{\hat{v}_{1,t}}}\odot (1-\beta_{11t})g_{1,t}  \right\rangle \mid\mathcal{F}^t\right]\right|\nonumber\\
    =&\eta(1-\beta_{11t})\frac{\alpha G_{2,\infty}}{\mu}\left\| \omega_{1,t}^T\omega_{2,t}-\omega_{1,*}^T\omega_{2,*}\right\|\mathop{{}\mathbb{E}}\left[\left\|\left(\omega_{1,t}^T\sigma_1\left(\omega_{2,t}z^t \right)-y^t\right)\sigma_2\left(\omega_{2,t}z^t \right)\right\|\mid\mathcal{F}^t\right]\nonumber\\
    \leq &\eta(1-\beta_{11t})\frac{\alpha G_{2,\infty}}{\mu}\left\| \omega_{1,t}^T\omega_{2,t}-\omega_{1,*}^T\omega_{2,*}\right\| \mathop{{}\mathbb{E}}\left[\left|\omega_{1,t}^T\sigma_1\left(\omega_{2,t}z^t \right)-\omega_{1,*}^T\sigma_1\left(\omega_{2,*}z^t \right)\right|\mid\mathcal{F}^t\right]\nonumber\\
    &\quad\quad\quad\quad\cdot\mathop{{}\mathbb{E}}\left[\left\|\sigma_2\left(\omega_{2,t}z^t \right)\right\|\mid\mathcal{F}^t\right]\nonumber\\
    =&\eta(1-\beta_{11t})\frac{\alpha G_{2,\infty}}{\mu}\left\| \omega_{1,t}^T\omega_{2,t}-\omega_{1,*}^T\omega_{2,*}\right\|\rho\left|\left(\omega_{1,t}^T\omega_{2,t}-\omega_{1,*}^T\omega_{2,*}\right)z^t\right|\rho\left\|\omega_{2,t}z^t\right\|\nonumber\\
    =&\eta(1-\beta_{11t})\frac{\alpha G_{2,\infty}}{\mu}\left\| \omega_{1,t}^T\omega_{2,t}-\omega_{1,*}^T\omega_{2,*}\right\|\rho^2\left|\left(\omega_{1,t}^T\omega_{2,t}-\omega_{1,*}^T\omega_{2,*}\right)z^t\right|\left\|\omega_{2,t}z^t\right\|\nonumber\\
    \leq&\eta(1-\beta_{11t})\frac{\alpha G_{2,\infty}}{\mu}\rho^2\left\|\omega_{2,t} \right\|\left(\omega_{1,t}^T\omega_{2,t}z^t-\omega_{1,*}^T\omega_{2,*}z^t\right)^2\frac{\left\| \omega_{1,t}^T\omega_{2,t}-\omega_{1,*}^T\omega_{2,*}\right\|\left\|z^t\right\|}{\left|\omega_{1,t}^T\omega_{2,t}z^t-\omega_{1,*}^T\omega_{2,*}z^t \right|}\nonumber\\
    \leq &2\eta\frac{\alpha G_{2,\infty}(1-\beta_{11t})}{\mu\cos{\epsilon}}\mathop{{}\mathbb{E}}\left[l_t \mid\mathcal{F}^t\right].\label{thm4:item1.2result}
\end{align}
The last inequality follows by applying conditions \ref{as2:3} and \ref{as2:4} in Assumption \hyperref[as:2]{2}. Next, Let us deal with the term in (\ref{thm4:item2}). Based on step~\ref{dnnAdam:m2} in Algorithm \ref{dnn ADAM}, we observe
\begin{align}
   &\mathop{{}\mathbb{E}}\left[\left\langle \left(\sqrt[2]{\hat{v}_{2,t}}\odot\omega_{2,t}\right)^T\omega_{1,t}-\left(\sqrt[2]{\hat{v}_{2,t}}\odot\omega_{2,*}\right)^T\omega_{1,*},\left(\frac{\eta}{\sqrt{\hat{v}_{2,t}}}\odot m_{2,t} \right)^T\omega_{1,t} \right\rangle \mid\mathcal{F}^t\right]\nonumber\\
   =&\eta\mathop{{}\mathbb{E}}\left[\left\langle \omega_{2,t}^T\omega_{1,t}-\omega_{2,*}^T\omega_{1,*}, m_{2,t}^T\omega_{1,t} \right\rangle \mid\mathcal{F}^t\right]\nonumber\\
   =&\eta\mathop{{}\mathbb{E}}\left[\left\langle \omega_{2,t}^T\omega_{1,t}-\omega_{2,*}^T\omega_{1,*}, \left(\beta_{12t}m_{2,t-1}+\left(1-\beta_{12t}\right)g_{2,t} \right)^T\omega_{1,t} \right\rangle \mid\mathcal{F}^t\right]\nonumber\\
   =&\eta \left[ \beta_{12t}\left\langle \omega_{2,t}^T\omega_{1,t}-\omega_{2,*}^T\omega_{1,*}, m_{2,t-1}^T\omega_{1t}\right\rangle+\left(1-\beta_{12t}\right)\left\langle \omega_{2,t}^T\omega_{1,t}-\omega_{2,*}^T\omega_{1,*}, \mathop{{}\mathbb{E}}\left[g_{2,t}^T\mid\mathcal{F}^t\right]\omega_{1,t} \right\rangle  \right]\nonumber\\
   =&\eta \beta_{12t}\left\langle \omega_{2,t}^T\omega_{1,t}-\omega_{2,*}^T\omega_{1,*}, m_{2,t-1}^T\omega_{1t}\right\rangle+\eta\left(1-\beta_{12t}\right)\nonumber\\
   &\quad \cdot\left\langle \omega_{2,t}^T\omega_{1,t}-\omega_{2,*}^T\omega_{1,*}, \left(\mathop{{}\mathbb{E}}\left[\left(\omega_{1,t}^T\sigma_1\left(\omega_{2,t}z^t \right)-y^t\right)\omega_{1,t}\left(\sigma_2\left(z^t\right)\right)^T\mid\mathcal{F}^t\right]\right)^T\omega_{1,t} \right\rangle \nonumber\\
   =&\eta \beta_{12t}\left( \omega_{1,t}^T\omega_{2,t}-\omega_{1,*}^T\omega_{2,*}\right) m_{2,t-1}^T\omega_{1,t}\label{thm4:item2.1}\\
   &\quad+
   \eta\rho^2\left(1-\beta_{12t}\right)\left\langle \omega_{2,t}^T\omega_{1,t}-\omega_{2,*}^T\omega_{1,*}, 
   \left(\omega_{1,t}^T\omega_{2,t}z^t -\omega_{1,*}^T\omega_{2,*}z^t\right)z^t
   \omega_{1,t}^T\omega_{1,t} \right\rangle\label{thm4:item2.2}.
\end{align}
The last equality holds true due to (\ref{lem3:grad2}) in Lemma \ref{lem3}. By using the fact that $\omega_{1,t},\omega_{2,t},\omega_{1,*},\omega_{2,*}$ and $m_{2,t-1}$ are all bounded, for the term in (\ref{thm4:item2.1}), there exists a constant $M_3$ such that
\begin{align}
    \left|\eta \beta_{12t}\left( \omega_{1,t}^T\omega_{2,t}-\omega_{1,*}^T\omega_{2,*}\right) m_{2,t-1}^T\omega_{1,t}\right|\leq \eta\beta_{12t}M_3\label{thm4:itme2.1result}.
\end{align}
At the same time, by inserting (\ref{obj:non}) from Lemma \ref{lem:obj} into (\ref{thm4:item2.2}) we get
\begin{align}
    &\eta\rho^2\left(1-\beta_{12t}\right)\left\langle \omega_{2,t}^T\omega_{1,t}-\omega_{2,*}^T\omega_{1,*}, 
   \left(\omega_{1,t}^T\omega_{2,t}z^t -\omega_{1,*}^T\omega_{2,*}z^t\right)z^t
   \omega_{1,t}^T\omega_{1,t} \right\rangle\nonumber\\
   =&\eta\rho^2(1-\beta_{12t})\left(\omega_{1,t}^T\omega_{2,t}z^t -\omega_{1,*}^T\omega_{2,*}z^t\right)\left(\omega_{1,t}^T\omega_{2,t}-\omega_{1,*}^T\omega_{2,*}\right)z^t\omega_{1,t}^T\omega_{1,t}\nonumber\\
   =&2(1-\beta_{12t})\eta\left\|\omega_{1,t}\right\|^2\mathop{{}\mathbb{E}}\left[l_t \mid\mathcal{F}^t\right]\nonumber\\
   \geq& 2(1-\beta_{121})\eta\left\|\omega_{1,t}\right\|^2\mathop{{}\mathbb{E}}\left[l_t \mid\mathcal{F}^t\right]
   \label{thm4:item2.2result} .
\end{align}
By inserting (\ref{thm4:item3+4}),(\ref{thm4:item1.1result}),(\ref{thm4:item1.2result}), (\ref{thm4:itme2.1result}),  and (\ref{thm4:item2.2result}),  into (\ref{thm4:1}) we obtain
\begin{align*}
    &2\mathop{{}\mathbb{E}}\left[l_t \mid\mathcal{F}^t\right]\left((1-\beta_{121})\left\|\omega_{1,t}\right\|^2 -  \frac{\alpha G_{2,\infty}(1-\beta_{11t})}{\mu\cos{\epsilon}}
    \right)\\
    \leq&\frac{1}{\eta}\left\{\mathop{{}\mathbb{E}}\left[\left\|\left(\sqrt[4]{\hat{v}_{2,t}}\odot\omega_{2,t+1}\right)^T\omega_{1,t+1}-\left(\sqrt[4]{\hat{v}_{2,t}}\odot\omega_{2,*}\right)^T\omega_{1,*} \right\|^2\mid\mathcal{F}^t \right]\right.\\
    &\quad\quad\left.-\mathop{{}\mathbb{E}}\left[\left\|\left(\sqrt[4]{\hat{v}_{2,t}}\odot\omega_{2,t}\right)^T\omega_{1,t}-\left(\sqrt[4]{\hat{v}_{2,t}}\odot\omega_{2,*}\right)^T\omega_{1,*} \right\|^2\mid\mathcal{F}^t \right]\right\}\\
    &\quad\quad+2\left(\beta_{11t}M_2+\beta_{12t}M_3\right)+\eta M_1.
\end{align*}
Since $\left\|\omega_{1,t} \right\| =\sqrt{\left[\frac{1}{2}+\xi_2\right]\mathbin{/}\left(1-\beta_{121}\right)}= \sqrt{\left[\frac{1}{2}+\frac{\alpha G_{2,\infty}}{\mu\cos{(\epsilon)}}\right]\mathbin{/}\left(1-\beta_{121}\right)}$, which in turn yields 
\begin{align*}
    2\left((1-\beta_{121})\left\|\omega_{1,t}\right\|^2 - \frac{\alpha G_{2,\infty}(1-\beta_{11t})}{\mu\cos{\epsilon}} \right)\geq 1.
\end{align*}
Therefore, by recalling the law of iterated expectations and summing up all loss functions for $t=p,p+1,\cdots,p+T$, we get
\begin{align}
    \sum_{t=p}^{T+p}\mathop{{}\mathbb{E}}\left[l_t \right]&\leq \frac{1}{\eta}\sum_{t=p}^{p+T}\left\{\mathop{{}\mathbb{E}}\left[\left\|\left(\sqrt[4]{\hat{v}_{2,t}}\odot\omega_{2,t+1}\right)^T\omega_{1,t+1}-\left(\sqrt[4]{\hat{v}_{2,t}}\odot\omega_{2,*}\right)^T\omega_{1,*} \right\|^2 \right]\right.\nonumber\\
    &\quad\quad\left.-\mathop{{}\mathbb{E}}\left[\left\|\left(\sqrt[4]{\hat{v}_{2,t}}\odot\omega_{2,t}\right)^T\omega_{1,t}-\left(\sqrt[4]{\hat{v}_{2,t}}\odot\omega_{2,*}\right)^T\omega_{1,*} \right\|^2 \right]\right\}\nonumber\\
    &\quad\quad+2\sum_{t=p}^{p+T}\left(\beta_{11t}M_2+\beta_{12t}M_3\right)+T\eta M_1.
    \label{thm4:2}
\end{align}
Applying the definition of $\beta_{11t}$ and $\beta_{12t}$ implies
\begin{align}
    &\sum_{t=p}^{p+T}\left(\beta_{11t}M_2+\beta_{12t}M_3\right) = \sum_{t=p}^{p+T}\left(\beta_{111}\gamma_1^tM_2+\beta_{121}\gamma_2^tM_3\right)\nonumber\\
    =&\beta_{111}M_2\sum_{t=p}^{p+T}\gamma_1^t+\beta_{121}M_3\sum_{t=p}^{p+T}\gamma_2^t
    \leq \frac{\beta_{111}M_2}{1-\gamma_1}+\frac{\beta_{121}M_3}{1-\gamma_2}.\label{thm4:3}
\end{align}
Since $z\in\mathbb{R}^d$, we notice that $\tilde{v}_{2,t}\in\mathbb{R}^{d\times d}$. Applying Lemma \ref{lem4} yields
\begin{align}
    &\sum_{t=p}^{p+T}\left\{\mathop{{}\mathbb{E}}\left[\left\|\left(\sqrt[4]{\hat{v}_{2,t}}\odot\omega_{2,t+1}\right)^T\omega_{1,t+1}-\left(\sqrt[4]{\hat{v}_{2,t}}\odot\omega_{2,*}\right)^T\omega_{1,*} \right\|^2 \right]\right.\nonumber\\
    &\quad\left.-\mathop{{}\mathbb{E}}\left[\left\|\left(\sqrt[4]{\hat{v}_{2,t}}\odot\omega_{2,t}\right)^T\omega_{1,t}-\left(\sqrt[4]{\hat{v}_{2,t}}\odot\omega_{2,*}\right)^T\omega_{1,*} \right\|^2 \right]\right\}\nonumber\\
    =&\sum_{t=p}^{p+T}\left\{\mathop{{}\mathbb{E}}\left[\left\|\left(\sqrt[4]{\tilde{v}_{2,t}}\right)^T\left(\omega_{2,t+1}^T\omega_{1,t+1}-\omega_{2,*}^T\omega_{1,*}\right) \right\|^2 \right]\right.\nonumber\\
    &\quad\left.-\mathop{{}\mathbb{E}}\left[\left\|\left(\sqrt[4]{\tilde{v}_{2,t}}\right)^T\left(\omega_{2,t}^T\omega_{1,t}-\omega_{2,*}^T\omega_{1,*}\right) \right\|^2 \right]\right\}\nonumber\\
    =&\sum_{t=p}^{p+T}\left\{\mathop{{}\mathbb{E}}\left[\sum_{i=1}^d\left[\sqrt{\tilde{v}_{2,t}}\right]_i\left[\omega_{2,t+1}^T\omega_{1,t+1}-\omega_{2,*}^T\omega_{1,*}\right]_i^2  \right]\right.\nonumber\\
    &\quad\left.-\mathop{{}\mathbb{E}}\left[\sum_{i=1}^d\left[\sqrt{\tilde{v}_{2,t}}\right]_i\left[\omega_{2,t}^T\omega_{1,t}-\omega_{2,*}^T\omega_{1,*}\right]_i^2  \right]\right\}\nonumber\\
    =&\mathop{{}\mathbb{E}}\left[\sum_{i=1}^d\left[\sqrt{\tilde{v}_{2,p}}\right]_i\left[\omega_{2,p}^T\omega_{1,p}-\omega_{2,*}^T\omega_{1,*}\right]_i^2\right]+\sum_{t=p+1}^{T+p}\left\{
    \mathop{{}\mathbb{E}}\left[\sum_{i=1}^d\left[\sqrt{\tilde{v}_{2,t}}\right]_i\left[\omega_{2,t}^T\omega_{1,t}-\omega_{2,*}^T\omega_{1,*}\right]_i^2  \right]\right.\nonumber\\
    &\quad\left.-\mathop{{}\mathbb{E}}\left[\sum_{i=1}^d\left[\sqrt{\tilde{v}_{2,t-1}}\right]_i\left[\omega_{2,t}^T\omega_{1,t}-\omega_{2,*}^T\omega_{1,*}\right]_i^2  \right]
    \right\}\nonumber\\
    =&\mathop{{}\mathbb{E}}\left[\sum_{i=1}^d\left[\sqrt{\tilde{v}_{2,p}}\right]_i\left[\omega_{2,p}^T\omega_{1,p}-\omega_{2,*}^T\omega_{1,*}\right]_i^2\right]\nonumber\\
    &\quad+\sum_{t=p+1}^{T+p}\sum_{i=1}^d\left\{
    \mathop{{}\mathbb{E}}\left[\left[\sqrt{\tilde{v}_{2,t}}\right]_i\left[\omega_{2,t}^T\omega_{1,t}-\omega_{2,*}^T\omega_{1,*}\right]_i^2 -\left[\sqrt{\tilde{v}_{2,t-1}}\right]_i\left[\omega_{2,t}^T\omega_{1,t}-\omega_{2,*}^T\omega_{1,*}\right]_i^2  \right]
    \right\}\nonumber\\
    =&\mathop{{}\mathbb{E}}\left[\sum_{i=1}^d\left[\sqrt{\tilde{v}_{2,p}}\right]_i\left[\omega_{2,p}^T\omega_{1,p}-\omega_{2,*}^T\omega_{1,*}\right]_i^2\right]\nonumber\\
    &\quad+
    \sum_{t=p+1}^{T+p}\sum_{i=1}^d
    \mathop{{}\mathbb{E}}\left[\left[\left(\sqrt{\tilde{v}_{2,t}}-\sqrt{\tilde{v}_{2,t-1}}\right)\right]_i\left[\omega_{2,t}^T\omega_{1,t}-\omega_{2,*}^T\omega_{1,*}\right]_i^2\right]\label{thm4:4},
\end{align}
where $\left[\sqrt{\tilde{v}_{2,t}}\right]_i$ represents the $i_{\mathrm{th}}$ element on diagonal in matrix $\tilde{v}_{2,t}$ and $\left[\omega_{2,t}^T\omega_{1,t}-\omega_{2,*}^T\omega_{1,*}\right]_i$ represents the $i_{\mathrm{th}}$ coordinate in vector $\omega_{2,t}^T\omega_{1,t}-\omega_{2,*}^T\omega_{1,*}$. Since $\omega_{1,t},\omega_{2,t},\omega_{1,*}$ and $\omega_{2,*}$ are all bounded for any $t$, e.g. $\left|\left[\omega_{2,t}^T\omega_{1,t}-\omega_{2,*}^T\omega_{1,*}\right]_i^2 \right|\leq W_{\infty}$ and $\tilde{v}_{2,t}\geq \tilde{v}_{2,t-1}$ due to the fact that $\hat{v}_{2,t}\geq \hat{v}_{2,t-1}$, (\ref{thm4:4}) can be further simplified as 
\begin{align}
    &\sum_{t=p}^{p+T}\left\{\mathop{{}\mathbb{E}}\left[\left\|\left(\sqrt[4]{\hat{v}_{2,t}}\odot\omega_{2,t+1}\right)^T\omega_{1,t+1}-\left(\sqrt[4]{\hat{v}_{2,t}}\odot\omega_{2,*}\right)^T\omega_{1,*} \right\|^2 \right]\right.\nonumber\\
    &\quad\left.-\mathop{{}\mathbb{E}}\left[\left\|\left(\sqrt[4]{\hat{v}_{2,t}}\odot\omega_{2,t}\right)^T\omega_{1,t}-\left(\sqrt[4]{\hat{v}_{2,t}}\odot\omega_{2,*}\right)^T\omega_{1,*} \right\|^2 \right]\right\}\nonumber\\
    \leq& W_{\infty}\sum_{i=1}^d\mathop{{}\mathbb{E}}\left[\left[\sqrt{\tilde{v}_{2,p}}\right]_i\right]+W_{\infty}\sum_{t=p+1}^{T+p}\sum_{i=1}^d
    \mathop{{}\mathbb{E}}\left[\left[\left(\sqrt{\tilde{v}_{2,t}}-\sqrt{\tilde{v}_{2,t-1}}\right)\right]_i\right]\nonumber\\
    =&W_{\infty}\sum_{i=1}^d\mathop{{}\mathbb{E}}\left[\left[\sqrt{\tilde{v}_{2,p+T}}\right]_i\right]\label{thm4:5}.
\end{align}
Substituting (\ref{thm4:3}) and (\ref{thm4:5}) in (\ref{thm4:2}) gives
\begin{align}
     \sum_{t=p}^{T+p}\mathop{{}\mathbb{E}}\left[l_t \right]\leq \frac{1}{\eta}W_{\infty}\sum_{i=1}^d\mathop{{}\mathbb{E}}\left[\left[\sqrt{\tilde{v}_{2,p+T}}\right]_i\right] +2\left(\frac{\beta_{111}M_2}{1-\gamma_1}+\frac{\beta_{121}M_3}{1-\gamma_2}\right) +T\eta M_1=\mathcal{O}(\sqrt{T}).\label{thm4:6}
\end{align}
The last equality uses the definition of $\eta=\frac{\eta_1}{\sqrt{T}}$. The desired result in Theorem \ref{thm3} follows directly from (\ref{thm4:6}) since it holds for any $p$.
\end{proof}

\end{document}